\documentclass[12pt]{colt2020_arxiv} 
\usepackage{times}

\usepackage{amsmath,amsfonts,bm}

\def\figref#1{figure~\ref{#1}}

\def\secref#1{section~\ref{#1}}

\def\eqref#1{equation~\ref{#1}}

\def\algref#1{algorithm~\ref{#1}}

\def\1{\bm{1}}

\def\vp{{\bm{p}}}

\def\vv{{\bm{v}}}
\def\vw{{\bm{w}}}
\def\vx{{\bm{x}}}
\def\vy{{\bm{y}}}
\def\vz{{\bm{z}}}

\def\mA{{\bm{A}}}
\def\mB{{\bm{B}}}

\def\mM{{\bm{M}}}

\def\mV{{\bm{V}}}
\def\mW{{\bm{W}}}

\DeclareMathAlphabet{\mathsfit}{\encodingdefault}{\sfdefault}{m}{sl}
\SetMathAlphabet{\mathsfit}{bold}{\encodingdefault}{\sfdefault}{bx}{n}
\newcommand{\tens}[1]{\bm{\mathsfit{#1}}}

\def\tV{{\tens{V}}}
\def\tW{{\tens{W}}}

\DeclareMathOperator{\sign}{sign}

\usepackage{hyperref}
\usepackage{url}
\usepackage{algorithm}
\usepackage{algorithmic}
\usepackage{amsfonts}
\usepackage{multirow}
\usepackage{amsmath}
\usepackage{graphicx}
\usepackage{amssymb}
\usepackage{subcaption}
\usepackage{natbib}
\usepackage{afterpage}
\usepackage{tikz}
\usepackage{wrapfig}
\usepackage{pgfplots}
\usetikzlibrary{arrows}
\usepgfplotslibrary{groupplots}
\usetikzlibrary{decorations.pathmorphing}
\usetikzlibrary{decorations.fractals}
\usepackage{xcolor}

\usepackage{rotating}
\usepackage{array}
\usepackage{multirow}

\usepackage{setspace}

\usepackage{amsmath}

\usepackage{boxedminipage}

\usepackage[export]{adjustbox}

\newcolumntype{L}{>{\centering\arraybackslash} m{0.04\columnwidth}} 
\newcolumntype{R}{>{\centering\arraybackslash} m{0.48\columnwidth}} 
\newcolumntype{S}{>{\centering\arraybackslash} m{0.32\columnwidth}} 

\newdimen\nodeDist
\usepackage{pifont}
\usetikzlibrary{arrows,decorations.pathmorphing,backgrounds,fit}
\usetikzlibrary{positioning} 
\usetikzlibrary{matrix}

\newtheorem{assumption}{Assumption}
\newtheorem{property}{Property}

\newcommand{\lemref}[1]{Lemma~\ref{#1}}
\newcommand{\thmref}[1]{Theorem~\ref{#1}}
\newcommand{\crlref}[1]{Corollary~\ref{#1}}

\title[Learning Circuits with Neural Networks]{Learning Boolean Circuits with Neural Networks}

\coltauthor{%
 \Name{Eran Malach} \Email{eran.malach@mail.huji.ac.il}\\
 \addr School of Computer Science, The Hebrew University, Israel
 \AND
 \Name{Shai Shalev-Shwartz} \Email{shais@cs.huji.ac.il}\\
 \addr School of Computer Science, The Hebrew University, Israel
}

\newcommand{\naturals}{\mathbb{N}}
\newcommand{\reals}{\mathbb{R}}
\newcommand{\inner}[1]{\langle #1 \rangle}
\newcommand{\norm}[1]{\left\| #1 \right\|}
\newcommand{\mean}[2]{\mathbb{E}_{#1} \left[ #2 \right]}

\newcommand{\abs}[1]{\left \lvert #1 \right \rvert}

\newcommand{\integers}{\mathbb{Z}}

\newcommand{\ignore}[1]{}
\newcommand{\D}[1]{\mathcal{D}^{(#1)}}
\newcommand{\Dt}[1]{\Psi(\mathcal{D})}

\newcommand{\prob}[2]{\mathbb{P}_{#1}\left[#2\right]}

\newcommand{\rank}{\text{rank}}

\begin{document}

\maketitle

\begin{abstract}
  While on some natural distributions, neural-networks are trained efficiently using gradient-based algorithms, it is known that
  learning them is computationally hard in the worst-case. To separate
  hard from easy to learn distributions, we observe the property of
  \textit{local correlation}: correlation between local patterns of
  the input and the target label. We focus on learning deep
  neural-networks using a gradient-based algorithm, when the target
  function is a tree-structured Boolean circuit. We show that in this
  case, the existence of \textit{correlation} between the gates of the
  circuit and the target label determines whether the optimization
  succeeds or fails. Using this result, we show that neural-networks
  can learn the $(\log n)$-parity problem for most product
  distributions. These results hint that \textit{local correlation} may
  play an important role in separating easy/hard to learn
  distributions. We also obtain a novel depth separation result, in
  which we show that a shallow network cannot express some
  functions, while there exists an \emph{efficient} gradient-based
  algorithm that can learn the very same functions using a deep
  network. The negative expressivity result for shallow networks is
  obtained by a reduction from results in communication complexity,
  that may be of independent interest. 
\end{abstract}

\section{Introduction and Motivation}
\label{sec:introduction}
It is well known (e.g. \cite{livni2014computational}) that while deep neural-networks
can \textbf{express} any function that can be run efficiently on a
computer, in the general case, \textbf{learning} neural-networks is
computationally hard. Despite this theoretic pessimism, in practice,
deep neural networks are successfully trained on real world
datasets. Bridging this theoretical-practical gap seems to be the holy
grail of theoretical machine learning nowadays. 
Maybe the most natural direction to bridge this gap is to find a
\textbf{property} of data distributions that determines whether
learning them is computationally easy or hard. The goal of this paper is to
propose such a property. 

To motivate this, we first recall the $k$-parity problem: the input is
$n$ bits, there is a subset of $k$ relevant bits (which are unknown to
the learner), and the output should be $1$ if the number of $-1$'s
among the relevant bits is even and $-1$ otherwise. It is well known
(e.g. \cite{shalev2017failures}) that the parity problem can be
expressed by a fully connected two layer network or by a depth $\log(n)$
locally connected~\footnote{i.e. every two adjacent neurons are only
  connected to one neuron in the upper layer.}  network. We observe
the behavior of a one hidden-layer neural network trained on the
$k$-parity problem, in two different instances: first, when the
underlying distribution is the uniform distribution (i.e. the
probability to see every bit is $\frac{1}{2}$); and second, when the
underlying distribution is a slightly biased product distribution (the
probability for every bit to be $1$ is $0.6$). As can be seen 
in \figref{fig:parity}, adding a slight bias to the probability of
each bit dramatically affects the behavior of the network: while on
the uniform distribution the training process completely fails, in the
biased case it converges to a perfect solution.

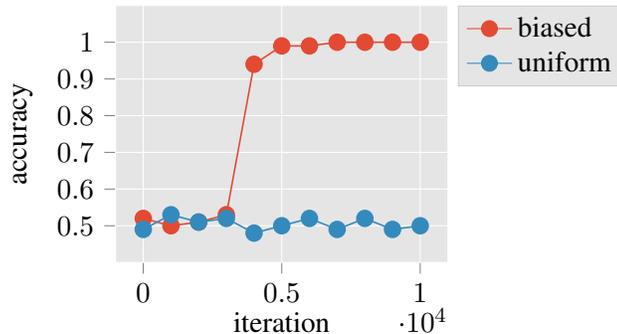
\begin{wrapfigure}{r}{0.6\textwidth}
\begin{center}
\begin{tikzpicture}

\definecolor{color1}{rgb}{0.203921568627451,0.541176470588235,0.741176470588235}
\definecolor{color0}{rgb}{0.886274509803922,0.290196078431373,0.2}
\definecolor{color3}{rgb}{0.984313725490196,0.756862745098039,0.368627450980392}
\definecolor{color2}{rgb}{0.596078431372549,0.556862745098039,0.835294117647059}

\begin{axis}[
axis background/.style={fill=white!89.80392156862746!black},
axis line style={white},
height=5cm,
width=6cm,
legend cell align={left},
legend entries={{biased},{uniform}},
legend pos=outer north east,
legend style={draw=white!80.0!black, fill=white!89.80392156862746!black},
tick align=outside,
tick pos=left,
x grid style={white},
xlabel={iteration},
xmajorgrids,
y grid style={white},
ylabel={accuracy},
ymajorgrids,
ymin=0.4, ymax=1.1,
ytick={0.5,0.6,0.7,0.8,0.9,1.0}
]
\addlegendimage{color0,  mark=*, mark size=3}
\addlegendimage{color1,  mark=*, mark size=3}
\addplot [semithick, color0, mark=*, mark size=3, mark options={solid}]
table [row sep=\\]{%
0		0.52 \\
1000		0.50 \\
2000		0.51 \\
3000		0.53 \\
4000		0.94 \\
5000		0.99 \\
6000		0.99 \\
7000		1.00 \\
8000		1.00 \\
9000		1.00 \\
10000	1.00 \\
};
\addplot [semithick, color1, mark=*, mark size=3, mark options={solid}]
table [row sep=\\]{%
0		0.49 \\
1000		0.53 \\
2000		0.51 \\
3000		0.52 \\
4000		0.48 \\
5000		0.50 \\
6000		0.52 \\
7000		0.49 \\
8000		0.52 \\
9000		0.49 \\
10000	0.50 \\
};
\end{axis}

\end{tikzpicture}
\end{center}
\caption{Trainig depth-two ReLU networks of size 128 with Adam, on both instances of the $k$-Parity problem ($k=5, n=128$). The figure shows the accuracy on a test set.} \label{fig:parity}
\end{wrapfigure}

This simple experiment shows that a small change in the underlying distribution can cause a dramatic change in the trainability of neural-networks. A key property that differentiates the uniform from the biased distribution is the \textit{correlation} between input bits and the target label. While in the uniform distribution, the correlation between each bit and the label  is zero, in the biased case every bit of the $k$ bits in the parity has a non-negligible correlation to the label (we show this formally in \secref{sec:distributions}). So, \textit{local correlations} between bits of the input and the target label seems to be a promising \textbf{property} which separates easy and hard distributions.

In this paper, we analyze the problem of learning tree-structured
Boolean circuits with neural-networks.  The key property that we
assume is having sufficient correlation between every influencing gate in the
circuit and the label.  We show a gradient-based algorithm that can
efficiently learn such circuits for distributions that satisfies the
Local Correlation Assumption (LCA). On the other hand, without
LCA, \textbf{any} gradient-based algorithm fails to learn
these circuits efficiently. We discuss specific
target functions and distributions that satisfy LCA. We show that for most product
distributions, our gradient-based algorithm learns the
$(\log n)$-parity problem (parity on $\log n$ bits of an input with
dimension $n$). We further show that for every circuit with AND/OR/NOT
gates, there exists a generative distribution, such that our algorithm
recovers the Boolean circuit exactly.

While our primary motivation for studying tree-structured Boolean circuits is as an extension to the family of $k$-Parities, we show that this family of functions has interesting properties in itself. Building on known results in communication complexity, we show that there exist tree-structured AND/OR circuits which cannot be expressed by any depth-two neural-network with bounded weights, unless an exponential number of units is used. Along with our positive results on learning tree-structured AND/OR circuits, this gives a family of distributions that: a) cannot be efficiently expressed by a shallow network and b) can be learned efficiently by a gradient-based algorithm using a deep network. To the best of our knowledge, this is the first result showing both depth separation and computationally efficient learnability for some family of distributions. As far as we are aware, all prior results on learning neural-networks with gradient-based algorithms assume a target function that can be expressed by a two-layer neural-network. Therefore, our result is the first theoretical result that motivates learning a deep network over a shallow one.

\section{Related Work}
In recent years, the success of neural-networks has inspired an ongoing theoretical research, trying to explain empirical observations about their behavior. Some theoretical works show failure cases of neural-networks. Other works give different guarantees on various learning algorithms for neural-networks. In this section, we cover the main works that are relevant to our paper.

\textbf{Failures of gradient-based algorithms}. Various works have shown different examples demonstrating failures of gradient-based algorithm. The work of \cite{shamir2018distribution} shows failures of gradient descent, both in learning natural target functions and in learning natural distributions. The work of \cite{shalev2017failures} shows that gradient-descent fails to learn parities and linear-periodic functions under the uniform distribution. In \cite{das2019learnability}, a hardness result for learning random deep networks is shown. Other similar failure cases are also covered in \cite{abbe2018provable, malach2019deeper}. While the details of these works differ, they all share the same key principal - if there is no local correlation, gradient-descent fails. Our work complements these results, showing that in some cases, when there are local correlations to the target, gradient-descent succeeds to learn the target function.

\textbf{Learning neural-networks with gradient-descent}. Recently, a large number of papers have provided positive results on learning neural-networks with gradient-descent. Generally speaking, most of these works show that over-parametrized neural-networks, deep or shallow, achieve performance that is competitive with kernel-SVM. \cite{daniely2017sgd} shows that SGD learns the conjugate kernel associated with the architecture of the network, for a wide enough neural-network. The work of \cite{brutzkus2017sgd} shows that SGD learns a neural-network with good generalization, when the target function is linear. A growing number of works show that for a specific kernel induced by the network activation, called the Neural Tangent Kernel (NTK), gradient-descent learns over-parametrized networks, for target functions with small norm in the reproducing kernel Hilbert space (see the works of \cite{jacot2018neural, xie2016diverse, oymak2018overparameterized, allen2018learning, allen2018convergence, oymak_towards_2019, arora2019fine, du2018gradient, ma2019comparative, lee2019wide}). While these results show that learning neural-networks with gradient-descent is not hopeless, they are in some sense disappointing --- in practice, neural-networks achieve performance that are far better than SVM, a fact that is not explained by these works. A few results do discuss success cases of gradient-descent that go beyond the kernel-based analysis \citep{brutzkus2017globally, brutzkus2019larger, allen2019can, yehudai2019power}. However, these works still focus on very simple cases, such as learning a single neuron, or learning shallow neural-networks in restricted settings. We deal with learning deep networks, going beyond the common reduction to linear classes of functions.

\textbf{Gradient-based optimization algorithms}. Deep networks are
optimized using an access to the objective function (population loss)
through a stochastic gradient oracle. The common algorithms that are
used in practice are variants of end-to-end stochastic gradient
descent, where gradients are calculated by backpropagation. However,
the focus of this paper is not on a specific algorithm but on the
question of what can be done with an access to a stochastic gradient
oracle. In particular, for simplicity of the analysis (as was also
done in \cite{arora2014provable, malach2018provably}), we analyze the
behavior of layerwise optimization --- optimizing one layer at a
time.  Recent works \citep{belilovsky2018greedy,
  belilovsky2019decoupled} have shown that layerwise training achieves
performance that is competitive with the standard end-to-end approach,
scaling up to the ImageNet dataset.

\textbf{Depth Separation}.
In recent years there has been a large number of works showing families of functions or distributions that are realizable by a deep network
of modest width,
but require exponential number of neurons to approximate by shallow networks.
Many works consider various measures of ``complexity'' that 
grow exponentially fast with the depth of the network, but not with
the width \citep{pascanu2013number, montufar2014number,
montufar2017notes, serra2017bounding, poole2016exponential,raghu2016expressive,telgarsky2016benefits}.
Recent works show depth separation results
for more ``natural'' families of functions.
For example, \cite{telgarsky2015representation} shows
a function on the real line that
exhibits depth separation.
The works of \cite{eldan2016power, safran2016depth,daniely2017depth} show
a depth separation argument for other natural functions, like
the indicator function of the unit ball. Other works by \cite{mhaskar2016deep, poggio2017and,delalleau2011shallow,cohen2016expressive} show that compositional
functions are better approximated 
by deep networks. All these results focus on the expressive power of neural-networks, but show no evidence that such functions can be learned using standard gradient-based algorithms. A recent work by \cite{malach2019deeper} shows that for  fractal distributions, gradient-descent fails when the depth separation is too strong (i.e., any shallow network achieves trivial performance on the distribution). Our work is the first positive learning result in this field, showing that for some family of distributions, a gradient-based algorithm succeeds in learning functions that cannot be efficiently expressed by shallow networks.

\textbf{Learning Boolean Circuits}. The problem of learning Boolean circuits has been studied in the classical literature of theoretical machine learning. The work of \cite{kearns1987learnability} gives various positive and negative results on the learnability of Boolean Formulas, including Boolean circuits. The work of \cite{linial1989constant} introduces an algorithm that learns a constant-depth circuit in quasi-polynomial time. Another work by \cite{kalai2018boolean} discusses various properties of learning Boolean formulas and Boolean circuits. Our work differs from the above in various aspects. Our main focus is learning deep neural-networks with gradient descent, where the target function is implemented by a Boolean circuit, and we do not aim to study the learnability of Boolean circuits in general. Furthermore, we consider Boolean circuits where a gate can take any Boolean functions, and not only AND/OR/NOT, as is often considered in the literature of Boolean circuits. On the other hand, we restrict ourselves to the problem of learning circuits with a fixed structure of full binary trees. We are not aware of any work studying a problem similar to ours.

\section{Problem Setting}
We consider the problem of learning binary classification functions over the Boolean cube.
So, let $\mathcal{X} = \{ \pm 1\}^n$ be the instance space and $\mathcal{Y} = \{\pm 1\}$ be the label set.
Throughout the paper, we assume the target function is given by a Boolean circuit. In general, such assumption effectively does not limit the set of target functions, as any computable function can be implemented by a Boolean circuit.
We define a circuit $C$ to be a directed graph with $n$ input nodes
and a single output node, where each inner node has exactly two incoming edges,
and is labeled by some arbitrary Boolean function
$f : \{\pm 1\}^2 \to \{\pm 1\}$, which we call a gate \footnote{Note that in the literature on Boolean circuits it is often assumed that the gates are limited to being AND/OR and NOT. We allow the gates to take any Boolean function, which makes this model somewhat stronger.}.
For each node $v$ in $C$ we denote by $\gamma(v) \in \left\lbrace f: \{\pm 1\}^2 \to \{\pm 1\} \right\rbrace$ its gate.
We recursively define $h_{v,C} : \{\pm 1\}^n \to \{\pm 1\}$ to be:
\[
h_{v,C}(\vx) = \gamma(v)\left(h_{u_1,C}(\vx), h_{u_2,C}(\vx)\right)
\]
where $u_1, u_2$ are the nodes with outcoming edges to $v$.
Define $h_{C} = h_{o,C}$, for the output node $o$.

We study the problem of learning the target function $h_C$, when $C$ is a full binary tree, and $n = 2^d$, where $d$ is the depth of the tree. The leaves of the tree are the input bits, ordered by $x_1, \dots x_n$.
We denote by $\mathcal{H}$ the family of functions that can be implemented by such tree-structured Boolean circuit.
Admittedly, $\mathcal{H}$ is much smaller than the family of all circuits of depth $d$, but still gives a rather rich family of functions. For example, $\mathcal{H}$ contains all the parity function on any $k$ bits of the input (the function calculated by $f(\vx) = \prod_{i \in I} x_i$ for some set of indexes $I$). We note that the the size of $\mathcal{H}$ grows like $6^n$, as shown in \cite{farhoodi2019functions}.

\begin{figure}
\centering
\tikzset{
  treenode/.style = {align=center, inner sep=2pt, text centered,
    font=\sffamily},
  arn_r/.style = {treenode, circle, draw=black, 
    text width=1.5em, very thick},
}
\begin{tikzpicture}[<-,>=stealth',level/.style={sibling distance = 5cm/#1,
  level distance = 0.8cm}] 
\node[arn_r]{$\gamma_{1,1}$}
    child{ node [arn_r] {$\gamma_{2,1}$} 
            child{ node [arn_r] {$\gamma_{3,1}$} 
            	child{ node {$x_1$}}
				child{ node {$x_2$}}
            }
            child{ node [arn_r] {$\gamma_{3,2}$}
							child{ node {$x_3$}}
							child{ node {$x_4$}}
            }                            
    }
    child{ node [arn_r] {$\gamma_{2,2}$}
            child{ node [arn_r] {$\gamma_{3,3}$} 
							child{ node {$x_5$}}
							child{ node {$x_6$}}
            }
            child{ node [arn_r] {$\gamma_{3,4}$}
							child{ node {$x_{n-1}$}}
							child{ node {$x_n$}}
            }
		}
; 
\end{tikzpicture}
{\caption{Tree-structured Boolean Circuit.}}
\end{figure}
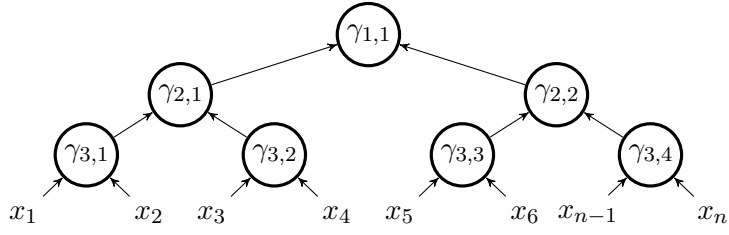

We introduce a few notations that are used in the sequel. Fix some tree-structured Boolean circuit $C$. This circuit has $d$ levels, and we denote $v_{i,j}$ the $j$-th node in the $i$-th
level of the tree, and denote $\gamma_{i,j} = \gamma(v_{i,j})$.
Fix some $i \in [d]$, let $n_i := 2^i$, and denote by
$\Gamma_i: \{\pm 1\}^{n_i} \to \{\pm 1\}^{n_i/2}$
the function calculated by the $i$-th level of the circuit:
\[
\Gamma_i(\vx) = \left(
\gamma_{i-1,1}(x_1, x_2), \dots, \gamma_{i-1, n_i/2}(x_{n_i-1}, x_{n_i})
\right)
\]
For $i<i'$, we denote:
$\Gamma_{i \dots i'} := \Gamma_i \circ \dots \circ \Gamma_{i'}$.
So, the full circuit is given by $h_C(\vx) = \Gamma_{1 \dots d}(\vx)$.

As noted, our goal is to learn Boolean circuits with neural-networks. To do so, we use a network architecture that aims to imitate the Boolean circuits described above. We replace each Boolean gate with a \textit{neural-gate}: a one hidden-layer ReLU network, with a hard-tanh\footnote{We chose to use the hard-tanh activation over the more popular tanh activation since it simplifies our theoretical analysis. However, we believe the same results can be given for the tanh activation.} activation on its output.
Formally, let $\sigma$ be the ReLU activation,
and let $\phi$ be the hard-tanh activation, so:
\[
\sigma(x) = \max (x,0), ~\phi(x) = \max(\min(x,1),-1)
\]
Define a \textit{neural-gate} to be a neural-network with one hidden-layer, input dimension 2,
with ReLU activation for the hidden-layer and hard-tanh for the output node.
So, denote $g_{\vw,\vv} : \reals^2 \to \reals$ s.t.:
\[
g_{\vw,\vv}(\vx) = \phi(\sum_{l=1}^k v_i\sigma(\inner{\vw_l, \vx}))
\]

Notice that a \textit{neural-gate} $g_{\vw, \vv}$ of width 4 or more can implement any Boolean gate. That is, we can replace any Boolean gate with a neural-gate, and maintain the same expressive power.
To implement the full Boolean circuit defined above, we construct a deep network of depth $d$ (the depth of the Boolean circuit), with the same structure as the Boolean circuit (see \figref{fig:fig_neural_circuit}). We define $d$ blocks, each block has \textit{neural-gates} with the same structure and connectivity as the Boolean circuit.
A block $B_{\mW^{(i)}, \mV^{(i)}} :\reals^{2^{i}} \to \reals^{2^{i-1}}$, is defined by:
\[
B_{\mW^{(i)}, \mV^{(i)}}(\vx) = [g_{\vw^{(i,1)},\vv^{(i,1)}}(x_1, x_2), g_{\vw^{(i,2)},\vv^{(i,2)}}(x_3, x_4), 
\dots, g_{\vw^{(i,2^{i-1})}, \vv^{(i,2^{i-1})}}(x_{2^i-1}, x_{2^i})]
\]
We consider the process of training neural-networks of the form $\mathcal{N}_{\tW, \tV} = B_{\mW^{(1)}, \mV^{(1)}} \circ \dots \circ B_{\mW^{(d)},\mV^{(d)}}$.
Notice that indeed, a network $\mathcal{N}_{\tW,\tV}$ can implement any tree-structured Boolean circuit of depth $d$.
We consider a layerwise optimization algorithm, that performs gradient updates layer-by-layer. Such approach has been recently shown to achieve performance that is comparable to end-to-end training, scaling up to the ImageNet dataset \citep{belilovsky2018greedy}.

\begin{figure}[t]
\floatconts
{fig:fig_neural_circuit}
{\caption{(a) Neural-gate. (b) Neural-network for learning the tree-structureed circuits.}}
{%
\subfigure{%
\label{fig:neural_gate}
\raisebox{0.5in}{\includegraphics[scale=0.4]{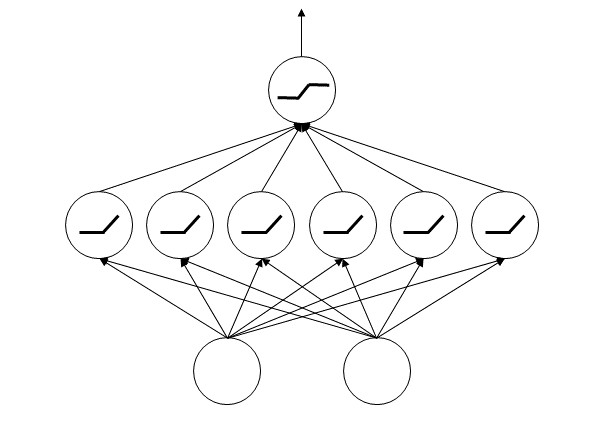}}
}\qquad 
\subfigure{%
\label{fig:neural_circuit}
\tikzset{
  treenode/.style = {align=center, inner sep=2pt, text centered,
    font=\sffamily},
  arn_r/.style = {treenode, 
    text width=3.5em},
}
\begin{tikzpicture}[<-,>=stealth',level/.style={sibling distance = 4.5cm/#1,
  level distance = 1.5cm}] 
\node[arn_r]{\includegraphics[width=1.0\textwidth,trim=30 0 30 0, clip]{neural_gate}}
    child{ node [arn_r] {\includegraphics[width=1.0\textwidth,trim=30 0 30 0, clip]{neural_gate}} 
            child{ node [arn_r] {\includegraphics[width=1.0\textwidth,trim=30 0 30 0, clip]{neural_gate}} 
            	child{ node {$x_1$}}
				child{ node {$x_2$}}
            }
            child{ node [arn_r] {\includegraphics[width=1.0\textwidth,trim=30 0 30 0, clip]{neural_gate}}
							child{ node {$x_3$}}
							child{ node {$x_4$}}
            }                            
    }
    child{ node [arn_r] {\includegraphics[width=1.0\textwidth,trim=30 0 30 0, clip]{neural_gate}}
            child{ node [arn_r] {\includegraphics[width=1.0\textwidth,trim=30 0 30 0, clip]{neural_gate}} 
							child{ node {$x_5$}}
							child{ node {$x_6$}}
            }
            child{ node [arn_r] {\includegraphics[width=1.0\textwidth,trim=30 0 30 0, clip]{neural_gate}}
							child{ node {$x_{n-1}$}}
							child{ node {$x_n$}}
            }
		}
; 
\end{tikzpicture}
}
}
\end{figure}

Denote by $P$ the average-pooling operator, defined by
$P(x_1, \dots, x_n) = \frac{1}{n} \sum_{i=1}^n x_i$.
To define the objective, we use the hinge-loss defined by $\ell(\hat{y},y) = \max(1-y\hat{y},0)$, and add a regularization term $R_\lambda(\hat{y}) = \lambda \abs{1-\hat{y}}$ to break symmetry in the optimization. We therefore denote the overall loss
on the distribution $\mathcal{D}$ and on a sample $S \subseteq \mathcal{X} \times \mathcal{Y}$ by:
\[
L_\mathcal{D}(f) = \mean{(\vx,y) \sim \mathcal{D}}{\ell(f(\vx),y) + R_\lambda(f(\vx))}, ~L_S(f) = \frac{1}{\abs{S}} \sum_{(\vx,y) \in S} \ell(f(\vx),y) + R_\lambda(f(\vx))
\]
The layerwise gradient-descent algorithm for learning deep networks is described in \algref{alg:bgd}.

\begin{algorithm}
   \caption{Layerwise Gradient-Descent}
   \label{alg:bgd}
\begin{algorithmic}
  \STATE \textbf{input}: 
\begin{ALC@g}
  \STATE Sample $S \subseteq \mathcal{X} \times \mathcal{Y}$, number of iterations $T \in \naturals$, learning rate $\eta \in \reals$.
\end{ALC@g}
  \STATE Let $\mathcal{N}_d \leftarrow id$
  \FOR{$i = d \dots 1$}
  \STATE Initialize $\mW^{(i)}_0, \mV^{(i)}_0$.
  \FOR{$t = 1 \dots T$}
  \STATE Update $\mW^{(i)}_t \leftarrow \mW^{(i)}_{t-1} - \eta \frac{\partial}{\partial \mW^{(i)}_{t-1}} L_{S}(P(B_{\mW^{(i)}_{t-1}, \mV^{(i)}_0} \circ \mathcal{N}_i))$
  \ENDFOR
  \STATE Update $\mathcal{N}_{i-1} \leftarrow B_{\mW^{(i)}_{T}, \mV^{(i)}_0} \circ \mathcal{N}_i$
  \ENDFOR
  \STATE Return $\mathcal{N}_0$
\end{algorithmic}
\end{algorithm}

For simplicity, we assume that the second layer of every \textit{neural-gate} is randomly initialized and fixed, such that $v \in \{\pm 1\}$. Notice that this does not limit the expressive power of the network. Algorithm \ref{alg:bgd} iteratively optimizes the output of the network's layers, starting from the bottom-most layer. For each layer, the average-pooling operator is applied to reduce the output of the layer to a single bit, and this output is optimized with respect to the target label. Note that in fact, we can equivalently optimize each \textit{neural-gate} separately and achieve the same algorithm. However, we present a layerwise training process to conform with algorithms used in practice.

\section{Main Results}
\label{sec:main_results}
Our main result shows that \algref{alg:bgd} can learn a function implemented by the circuit $C$, when running on ``nice'' distributions, with the local correlation property. We start by describing the distributional assumptions needed for our main results. Let $\mathcal{D}$ be some distribution over $\mathcal{X} \times \mathcal{Y}$.
For some function $f : \mathcal{X} \to \mathcal{X}'$, we denote by $f(\mathcal{D})$
the distribution of $(f(\vx), y)$ where $(\vx,y) \sim \mathcal{D}$.
Let $\D{i}$ be the distribution $\Gamma_{(i+1) \dots d}(\mathcal{D})$.
Denote by $c_{i,j}$ the correlation between the output of the $j$-th gate in the
$i$-th layer and the label, so:
$c_{i,j} := \mean{\D{i}}{x_j y}$.
Denote the influence of the $(i,j)$ gate with respect to the uniform distribution ($U$) by:
\[
\mathcal{I}_{i,j} := \prob{\vx \sim U}{\Gamma_{i-1}(\vx) \ne \Gamma_{i-1}(\vx \oplus e_j)} := 
\prob{\vx \sim U}{\Gamma_{i-1}(\vx) \ne \Gamma_{i-1}(x_1, \dots, -x_j, \dots, x_n)}
\]
We assume w.l.o.g. that for every $i,j$ such that $\mathcal{I}_{i,j} = 0$,
the $(i,j)$ gate is constant $\gamma_{i,j} \equiv 1$.
Since the output of this gate has no influence on the output of the circuit, we can choose it freely without changing the target function.
Again w.l.o.g., we assume $\mean{\mathcal{D}}{y} \ge 0$. Indeed, if this does not hold, we can observe the same distribution with swapped labels. Now, our main assumption is the following:
\begin{assumption} (LCA)
\label{asm:correlation_basic}
There exists some $\Delta \in (0,1)$ such that
for every layer $i \in [d]$ and for every gate $j \in [2^i]$ with $\mathcal{I}_{i,j} \ne 0$, the value of $c_{i,j}$ satisfies $\abs{c_{i,j}} > \mean{\mathcal{D}}{y} + \Delta$.
\end{assumption}

Essentially, this assumption requires that the output of every gate in the circuit will ``explain'' the label slightly better than simply observing the bias between positive and negative examples. Clearly, gates that have no influence on the target function never satisfy this property, so we require it only for influencing gates. While this is a strong assumption, in \secref{sec:distributions} we discuss examples of distributions where this assumption typically holds.

Now, we start by considering the case of non-degenerate product distributions:
\begin{definition} A distribution $\mathcal{D}$ is a $\Delta$-non-degenerate product distribution, if the following holds:
\begin{itemize}
\item For every $j \ne j'$, the variables $x_j$ and $x_{j'}$ are independent, for $(\vx, y) \sim \mathcal{D}$.
\item For every $j$, we have $\prob{(\vx, y) \sim \mathcal{D}}{x_j = 1} \in (\Delta, 1-\Delta)$.
\end{itemize} 
\end{definition}

Our first result shows that for non-degenerate product distributions satisfying LCA, our algorithm returns a network with zero loss w.h.p., with polynomial sample complexity and run-time:
\begin{theorem}
\label{thm:product_convergence_exact}
Fix $\delta, \Delta \in (0,\frac{1}{2})$ and integer $n=2^d$. Let
$k \ge \log^{-1}(\frac{4}{3}) \log(\frac{2nd}{\delta})$,
$\eta \le \frac{1}{16\sqrt{2}k}$ and $\lambda = \mean{}{y} + \frac{\Delta}{4}$.
Fix some $h \in \mathcal{H}$, and let $\mathcal{D}$ be a $\Delta$-non-degnerate product distribution separable by $h$ s.t. $\mathcal{D}$ satisfies assumption \ref{asm:correlation_basic} (LCA) with parameter $\Delta$.
Assume we sample $S \sim \mathcal{D}$, with $\abs{S} > \frac{2^{15}}{\Delta^6} \log(\frac{8nd}{\delta})$. 
Then, with probability $\ge 1-\delta$, when running \algref{alg:bgd} with initialization of $\mW$ s.t. $\norm{\mW^{(i)}_0}_{\max} \le \frac{1}{4\sqrt{2}k}$
on the sample $S$, the algorithm returns a function $\mathcal{N}_0$ such that $\prob{(\vx,y) \sim \mathcal{D}}{\mathcal{N}_0(\vx) \ne y} = 0$, when 
running $T > \frac{24n}{ \sqrt{2} \eta \Delta^3}$ steps for each layer.
\end{theorem}

In fact, we can show this result for a larger family of distributions, going beyond product distributions. First, we show that a non-degenerate product distribution $\mathcal{D}$ that satisfies LCA, satisfies the following properties:

\begin{property}
\label{prp:correlation}
There exists some $\Delta \in (0,1)$ such that
for every layer $i \in [d]$ and for every gate $j \in [2^i]$, the output of the $j$-th gate in the $i$-th layer satisfies one of the following:
\begin{itemize}
\item The value of the gate $j$ is independent of the label $y$, and its influence is zero: $\mathcal{I}_{i,j} = 0$.
\item The value of $c_{i,j}$ satisfies $\abs{c_{i,j}} > \mean{\mathcal{D}}{y} + \Delta$.
\end{itemize}
\end{property}
\begin{property}
\label{prp:cond_indep}
For every layer $i \in [d]$, and every gate $j \in [2^{i-1}]$,
the value of $(x_{2j-1}, x_{2j})$
(i.e, the input to the $j$-th gate of layer $i-1$) is independent of the label
$y$ given the output of the $j$-th gate:
\begin{align*}
&\prob{(x,y) \sim \D{i}}{(x_{2j-1},x_{2j}) = \vp, y = y' |
\gamma_{i-1,j}(x_{2j-1},x_{2j})} \\
&= \prob{(\vx,y) \sim \D{i}}{(x_{2j-1},x_{2j}) = \vp |
\gamma_{i-1,j}(x_{2j-1},x_{2j})} \cdot 
\prob{(\vx,y) \sim \D{i}}{y = y' |
\gamma_{i-1,j}(x_{2j-1},x_{2j})}
\end{align*}
\end{property}

\begin{property}
\label{prp:cond_bound}
There exists some $\epsilon \in (0,1)$ such that
for every layer $i \in [d]$, for every gate $j \in [2^{i-1}]$
and for every $\vp \in \{\pm 1\}^2$ such that $\prob{(\vx,y) \sim \D{i}}{(x_{2j-1},x_{2j}) = \vp} > 0$, it holds that:
$\prob{(\vx,y) \sim \D{i}}{(x_{2j-1},x_{2j}) = \vp} \ge \epsilon$.
\end{property}
From the following lemma, these properties generalize the case of product distributions with LCA:
\begin{lemma}
\label{lem:prod_dist}
Any $\Delta$-non-degenerate product distribution $\mathcal{D}$ satisfying assumption \ref{asm:correlation_basic} (LCA), satisfies properties \ref{prp:correlation}-\ref{prp:cond_bound}, with $\epsilon = \frac{\Delta^2}{4}$.
\end{lemma}

Notice that properties \ref{prp:correlation}, \ref{prp:cond_indep} and \ref{prp:cond_bound} may hold for distributions that are not product distribution (as we show in the next section). Specifically, property \ref{prp:cond_indep} is a very common assumption in the field of Graphical Models (see \cite{koller2009probabilistic}).
As in \thmref{thm:product_convergence_exact}, given a distribution satisfying properties \ref{prp:correlation}- \ref{prp:cond_bound}, we show that with high probability, \algref{alg:bgd} returns a function with zero loss, with sample complexity and run-time polynomial in the dimension $n$:

\begin{theorem}
\label{thm:generative_convergence_exact}
Fix $\delta, \Delta, \epsilon \in (0,\frac{1}{2})$ and integer $n=2^d$. Let
$k \ge \log^{-1}(\frac{4}{3}) \log(\frac{2nd}{\delta})$, $\eta \le \frac{1}{16\sqrt{2}k}$ and $\lambda = \mean{}{y} + \frac{\Delta}{4}$.
Fix some $h \in \mathcal{H}$, and let $\mathcal{D}$ be a distribution separable by $h$ which satisfies properties \ref{prp:correlation}-\ref{prp:cond_bound} with $\Delta, \epsilon$.
Assume we sample $S \sim \mathcal{D}$, with $\abs{S} > \frac{2^{11}}{\epsilon^2 \Delta^2} \log(\frac{8nd}{\delta})$. 
Then, with probability $\ge 1-\delta$, when running \algref{alg:bgd} with initialization of $\mW$ s.t. $\norm{\mW^{(i)}_0}_{\max} \le \frac{1}{4\sqrt{2}k}$
on the sample $S$, the algorithm returns a function such that $\prob{(\vx,y) \sim \mathcal{D}}{\mathcal{N}_0(\vx) \ne y} = 0$, when 
running $T > \frac{6n}{ \sqrt{2} \eta \epsilon \Delta}$ steps for each layer.
\end{theorem}

We give the full proof of the theorems in the appendix, and give a sketch of the argument here. First, note that \thmref{thm:product_convergence_exact} follows from \thmref{thm:generative_convergence_exact} and \lemref{lem:prod_dist}. To prove \thmref{thm:generative_convergence_exact}, observe that the input to the $(i,j)$-th \textit{neural-gate} is a pattern of two bits. The target gate (the $(i,j)$-th gate in the circuit $C$) identifies each of the four possible patterns with a single output bit. For example, if the gate is OR, then the patterns $\{(1,1), (-1,1), (1,-1)\}$ get the value $1$, and the pattern $(-1,-1)$ gets the value $-1$. Fix some pattern $\vp \in \{\pm 1\}^2$, and assume that the output of the $(i,j)$-th gate on the pattern $\vp$ is $1$. Since we assume the output of the gate is correlated with the label, the loss function draws the output of the \textit{neural-gate} on the pattern $\vp$ toward the \textit{correlation} of the gate. In the case where the output of the gate on $\vp$ is $-1$, the output of the \textit{neural-gate} is drawn to the opposite sign of the \textit{correlation}. All in all, the optimization separates the patterns that evaluate to $1$ from the patterns that evaluate to $-1$. In other words, the \textit{neural-gate} learns to implement the target gate. This way, we can show that the network recovers all the influencing gates, so at the end of the optimization process the network implements the circuit.

\subsection{Correlation is Necessary}
Observe that when there is no correlation, the above argument fails immediately. Without correlation, the output of the \textit{neural-gate} is drawn towards a constant value for all the input patterns, regardless of the value of the gate. If the gate is not influencing the target function (i.e. $\mathcal{I}_{i,j} = 0$), then this clearly doesn't effect the overall behavior. However, if there exists some influencing gate with no correlation to the label, then the output of the \textit{neural-gate} will be constant on all its input patterns. Hence, the algorithm will fail to recover the target function. This shows that LCA is in fact critical for the success of the algorithm.

But, maybe there exists a gradient-based algorithm, or otherwise a better network architecture, that can learn the function class $\mathcal{H}$, even without assuming local correlations. Relying on well-known results on learning parities with statistical-queries, we claim that this is not the case. In fact, no gradient-based algorithm can learn the function class $\mathcal{H}$ in the general case. This follows from the fact that $\mathcal{H}$ contains all the parity functions on $k$ bits, and these are hard to learn in the statistical-query model (see \cite{kearns1998efficient}).
Since any gradient-based algorithm with stochastic noise\footnote{Even an extremely small amount of noise, below the machine precision.} can be implemented using a statistical-query oracle (see \cite{feldman2017statistical}), this result implies that any gradient-based algorithm with stochastic noise fails to learn the class $\mathcal{H}$. Note that even if we observe only parities on $\log n$ bits, the number of gradient steps needed still grows super-polynomially with $n$ (namely, it grows like $n^{\Omega(\log n)}$). Similar results are derived in \cite{shalev2017failures}, analyzing gradient-descent directly for learning parity functions.

\section{Distributions that Satisfy LCA}
\label{sec:distributions}
In the previous section we showed that \algref{alg:bgd} can learn tree-structured Boolean circuits in polynomial run-time and sample complexity. These results require some non-trivial distributional assumptions. In this section we study specific families of distributions, and show that they satisfy the above assumptions.

First, we study the problem of learning a parity function on $\log n$ bits of the input, when the underlying distribution is a product distribution. The problem of learning parities was studied extensively in the literature of machine learning theory \citep{feldman2006new,feldman2009agnostic,blum2003noise,shalev2017failures, brutzkus2019id3}, and serves as a good case-study for the above results. In the $(\log n)$-parity problem, we show that in fact \textit{most} product distributions satisfy assumption \ref{asm:correlation_basic}, hence our results apply to most product distributions. Next, we study distributions given by a generative model. We show that for every circuit with gates AND/OR/NOT, there exists a distribution that satisfies the above assumptions, so \algref{alg:bgd} can learn any such circuit exactly.

\subsection{Parities and Product Distributions}
We observe the $k$-Parity problem, where the target function is $f(\vx) = \prod_{j \in I} x_j$ some subset $I \subseteq [n]$ of size $\abs{I} = k$.
A simple construction shows that $f$ can be implemented by a tree structured circuit as defined previously. We define the gates of the first layer by:
\[
\gamma_{d-1,j}(z_1,z_2) = \begin{cases}
z_1 z_2 & x_{2j-1}, x_{2j} \in I \\
z_1 & x_{2j-1} \in I, x_{2j} \notin I \\
z_2 & x_{2j} \in I, x_{2j-1} \notin I \\
1 & o.w
\end{cases}
\]
And for all other layers $i < d-1$, we define: $\gamma_{i,j}(z_1, z_1) = z_1 z_2$.
Then we get the following:
\begin{lemma}
\label{lem:parity_expressivity}
Let $C$ be a Boolean circuit as defined above. Then: 
$h_C(\vx) = \prod_{j \in I} x_j = f(\vx)$.
\end{lemma}

Now, let $\mathcal{D}_{\mathcal{X}}$ be some product distribution over $\mathcal{X}$, and denote $p_j := \prob{\mathcal{D}_{\mathcal{X}}}{x_j = 1}$.
Let $\mathcal{D}$ be the distribution of $(\vx, f(\vx))$
where $\vx \sim \mathcal{D}_{\mathcal{X}}$.
Then for the circuit defined above we get the following:
\begin{lemma}
\label{lem:parity_product_distribution}
Fix some $\xi \in (0,\frac{1}{4})$.
For every product distribution $\mathcal{D}$ with $p_j \in (\xi, \frac{1}{2}-\xi) \cup (\frac{1}{2} + \xi, 1-\xi)$ for all $j$, if $\mathcal{I}_{i,j} \ne 0$ then $\abs{c_{i,j}} - \abs{\mean{}{y}} \ge (2\xi)^{k}$ and $\prob{(\vz, y) \sim \Gamma_{(i+1) \dots d}(\mathcal{D})}{z_j = 1} \in (\xi, 1-\xi)$.
\end{lemma}

The above lemma shows that every non-degenerate product distribution that is far enough from the uniform distribution, satisfies assumption \ref{asm:correlation_basic} with $\Delta = (2\xi)^k$.
Using the fact that at each layer, the output of each gate is an independent random variable (since the input distribution is a product distribution), we get that property \ref{prp:cond_bound} is satisfied with $\epsilon = \xi^2$.
This gives us the following result:

\begin{corollary}
Let $\mathcal{D}$ be a product distribution with $p_j \in (\xi, \frac{1}{2} -\xi) \cup (\frac{1}{2} + \xi, 1-\xi)$ for every $j$, with the target function being a ($\log n$)-Parity (i.e., $k= \log n$). Then, when running \algref{alg:bgd} as described in \thmref{thm:generative_convergence_exact}, with probability at least $1-\delta$ the algorithm returns the true target function $h_C$, with run-time and sample complexity polynomial in $n$.
\end{corollary}

\subsection{Generative Models}
Next, we move beyond product distributions, and observe families of distributions given by a generative model.
We limit ourselves to circuits where each gate is chosen from the set $\{\wedge, \vee, \neg \wedge, \neg \vee\}$.
For every such circuit, we define a generative distribution as follows: we start by sampling a label for the example. Then iteratively, for every gate, we sample uniformly at random a pattern from all the pattern that give the correct output. For example, if the label is $1$ and the topmost gate is OR, we sample a pattern uniformly from $\{(1,1), (1,-1), (-1,1)\}$. The sampled pattern determines what should be the output of the second topmost layer. For every gate in this layer, we sample again a pattern that will result in the correct output. We continue in this fashion until reaching the bottom-most layer, which defines the observed example.
Formally, for a given gate $\Gamma \in \{\wedge, \vee, \neg \wedge, \neg \vee\}$,
we denote the following sets of patterns:
\[
S_{\Gamma} = \{\vv \in \{\pm 1\}^2 ~:~ \Gamma(v_1, v_2) = 1 \},~
S_{\Gamma}^c = \{\pm 1\}^2 \setminus S_{\Gamma}
\]
We recursively define
$\mathcal{D}^{(0)}, \dots, \mathcal{D}^{(d)}$, where $\mathcal{D}^{(i)}$
is a distribution over $\{\pm 1\}^{2^i} \times \{\pm 1\}$:
\begin{itemize}
\item $\D{0}$ is a distribution on $\{(1,1), (-1,-1)\}$ s.t. $\prob{\D{0}}{(1,1)} = \prob{\D{0}}{(-1,-1)} = \frac{1}{2}$.
\item To sample $(\vx, y) \sim \D{i}$, sample $(\vz,y) \sim \D{i-1}$.
Then, for all $j \in [2^{i-1}]$, if $z_j = 1$ sample
$\vx'_j \sim U(S_{\gamma_{i,j}})$, and otherwise sample
$\vx'_j \sim U(S^c_{\gamma_{i,j}})$. Set
$\vx = [\vx'_1, \dots, \vx'_{2^{i-1}}] \in \{\pm 1\}^{2^i}$, 
and return $(\vx, y)$.
\end{itemize}
Then we have the following results:
\begin{lemma}
\label{lem:correlation}
For every $i \in [d]$ and every $j \in [2^i]$, denote $c_{i,j} = \mean{(\vx,y) \sim \D{i}}{x_j y}$. Then we have:
\[
\abs{c_{i,j}} - \mean{}{y}
> \left( \frac{2}{3} \right)^d = n^{\log (2/3)}
\]
\end{lemma}
\begin{lemma}
\label{lem:gate_distribution}
For every $i \in [d]$ we have $\Gamma_i(\D{i}) = \D{i-1}$.
\end{lemma}

Notice that from \lemref{lem:correlation}, the distribution $\D{d}$ satisfies property \ref{prp:correlation} with $\Delta = n^{\log (2/3)}$ (note that since we restrict the gates to AND/OR/NOT, all gates have influence).
By its construction, the distribution also satisfies property \ref{prp:cond_indep}, and it satisfies property \ref{prp:cond_bound} with $\epsilon = \left(\frac{1}{4}\right)^d = \frac{1}{n^2}$. Therefore, we can apply \thmref{thm:generative_convergence_exact} on the distribution $\D{d}$, and get that \algref{alg:bgd} learns the circuit $C$ \textit{exactly} in polynomial time.
This leads to the following corollary:
\begin{corollary}
\label{crl:and_or_circuits}
With the assumptions and notations of \thmref{thm:generative_convergence_exact},
for every circuit $C$ with gates in $\{\wedge, \vee, \neg \wedge, \neg \vee\}$, there exists a distribution $\mathcal{D}$
such that when running \algref{alg:bgd} on a sample from $\mathcal{D}$, the algorithm returns $h_C$ with probability $1-\delta$, in polynomial run-time and sample complexity.
\end{corollary}

Note that the fact that for every circuit there exists a distribution that can be learned in the PAC setting is trivial: simply take a distribution that is concentrated on a single positive example, and approximating the target function on such distribution is achieved by a classifier that always returns a positive prediction.
However, showing that there exists a distribution on which \algref{alg:bgd} \textit{exactly} recovers the circuit, is certainly non-trivial.

\section{Depth Separation}
\label{sec:depth_separation}
In \secref{sec:main_results} we showed a gradient-based algorithm that
efficiently learns a network of depth $d$, which recovers functions
from $\mathcal{H}$ under some distributional assumption. We now show
that the family $\mathcal{H}$ of tree-structured Boolean circuits
contains functions that cannot be expressed by shallow (depth-two)
neural-networks with bounded weights, unless an exponential number of neurons is used. At the same time, from what we previously showed, these functions can be expressed by a $(\log n)$-depth
locally-connected network of moderate size, and can be learned efficiently using \algref{alg:bgd} (under some distribution). 

We start by introducing the notion of \textit{sign-rank} of a binary
matrix.  Let $\mA \in \{\pm 1\}^{m \times m}$ be some matrix with
entries in $\{\pm 1\}$.  The \textit{sign-rank} of $\mA$, denoted
$sr(\mA)$, is the least rank of a real matrix
$\mB \in \reals^{m \times m}$ such that $\sign \mB_{i,j} = \mA_{i,j}$.
Now, let $n' = \frac{n}{2}$ and let $\mathcal{X}' = \{\pm
1\}^{n'}$. For some
$f: \mathcal{X}' \times \mathcal{X}' \to \{\pm 1\}$ we denote the
matrix
$\mM = [f(\vx,\vy)]_{\vx,\vy \in \mathcal{X}'} \in \{\pm 1\}^{2^{n'}
  \times 2^{n'}}$.  Studying the \textit{sign-rank} of different functions has
attracted attention over the years, due to various implications in
communication complexity, circuit complexity and learning theory.  The
following result shows a function with \textit{sign-rank} that grows
exponentially with the dimension:
\begin{theorem} \citep{razborov2010sign}
\label{thm:sign_rank_ac0}
Let $m = \sqrt[3]{n'}$ and 
$f_m(\vx,\vy) = \wedge_{i=1}^m \vee_{j=1}^{m^2} (x_{ij} \wedge y_{ij})$.
Then $\mM = [f_m(x,y)]_{x,y}$ has $sr(\mM) = 2^{\Omega(m)}$.
\end{theorem}

Following an argument similar to \cite{forster2001relations}, we show that the sign-rank of a neural-network with bounded integer weights is only polynomial in the dimension $n$.
Using this, we show that $f_m(\vx,\vy)$ cannot be implemented by a depth-two neural network of polynomial size:

\begin{theorem}
\label{thm:depth_separation}
Let $g(\vx,\vy) = \sum_{i=1}^k u_i \sigma(\inner{\vw_i,\vx} + \inner{\vv_i, \vy} + b_i)$ for $\sigma(x) = \max \{x,0\}$ (ReLU) and bounded weights $\vw_i, \vv_i \in [-B,B]^{n'}$, bias $b_i \in [-B,B]$, and bounded second layer $u_i \in [-B,B]$.
If there exists a margin $\gamma$ such that $g(\vx,\vy) \cdot f_m(x,y) \ge \gamma$ for all $\vx,\vy \in \mathcal{X}$, then $k \ge 2^{\Omega(m)}$.
\end{theorem}

Now, it is immediate that the function $f_m$ can be calculated using a tree-structured Boolean circuits with AND/OR gates. In \crlref{crl:and_or_circuits}, we showed that for any such circuit there exists a distribution for which \algref{alg:bgd} exactly recovers the circuit. Therefore, we get the following:

\begin{corollary}
There exists a family $\mathcal{F}$ of distributions over $\mathcal{X} \times \mathcal{Y}$, such that:
\begin{enumerate}
\item For any $\mathcal{D} \in \mathcal{F}$, \algref{alg:bgd} returns a function which separates $\mathcal{D}$ with margin $1$ w.h.p., with polynomial runtime and sample-complexity.
\item There exists $\mathcal{D} \in \mathcal{F}$ such that any depth-two bounded-weight ReLU network which separates $\mathcal{D}$ with a constant margin $\gamma$, has size $2^{\Omega(n)}$.
\end{enumerate}
\end{corollary}

To the best of our knowledge, this is the first result that shows a family of distributions with both a positive result on efficient learnability using a deep network, and a negative result on efficient expressivity using a shallow network.

\section{Discussion}
In this paper we suggested the property of \textit{local corrleation} (LCA) as a possible candidate for differentiating between hard and easy distributions. We showed that on the task of learning tree-structured Boolean circuits, the existence of \textit{local correlations} between the gates and the target label allows layerwise gradient-descent to learn the target circuit. Furthermore, we showed specific tasks and distributions which satisfy the LCA. 
Admittedly, as the primary focus of this paper is on theoretical
analysis, the distributions we study are synthetic in nature. It is
interesting to see if LCA holds, to some extent, on natural
distributions as well. While a rigorous study of this question is out
of the scope of this paper, we have performed the following simple
experiment: we trained a network with two hidden-layers on a
\textit{single} random patch from images in the ImageNet dataset. We
observed that even on a complex task such as ImageNet, a network that
gets only a $3\times 3$ patch as an input, achieves $2.6\%$ top-5
accuracy --- much better than a random guess ($0.5\%$ top-5
accuracy). The full results of the experiment are detailed in the
appendix. This experiment gives some hope that LCA may be a relevant
property for studying the success of deep learning in practice.

Our results raise a few open questions, which we leave for future
work. The most immediate research problem is showing similar results
for more general structures of Boolean circuit, and on a wider range
of distributions (beyond product distributions or generative
models). More generally, we suggest that LCA may be important in a broader context, beyond
Boolean circuits. For example, examining whether an equivalent
property exists when the target function is a convolutional network is
an extremely interesting open problem. Needless to say, finding other properties of natural distribution that determine whether gradient-based algorithms succeed or fail is another promising research direction.

\paragraph{Acknowledgements:} This research is supported by the European Research Council (TheoryDL project).

\newpage

\bibliography{circuits}
\bibliographystyle{iclr2020_conference}

\newpage
\appendix

\section{Experiments}
Figure \ref{fig:imagenet} details the results of the ImageNet experiment discussed in the introduction.
\begin{figure}[H]
\begin{center}
\begin{tikzpicture}

\definecolor{color1}{rgb}{0.203921568627451,0.541176470588235,0.741176470588235}
\definecolor{color0}{rgb}{0.886274509803922,0.290196078431373,0.2}
\definecolor{color3}{rgb}{0.984313725490196,0.756862745098039,0.368627450980392}
\definecolor{color2}{rgb}{0.596078431372549,0.556862745098039,0.835294117647059}

\begin{axis}[
axis background/.style={fill=white!89.80392156862746!black},
axis line style={white},
height=5cm,
width=6cm,
legend cell align={left},
legend entries={{Top-5 (random)},{Top-5 (single-patch)},{Top-1 (random)},{Top-1 (single-patch)}},
legend pos=outer north east,
legend style={draw=white!80.0!black, fill=white!89.80392156862746!black},
tick align=outside,
tick pos=left,
x grid style={white},
xlabel={patch size (px)},
xmajorgrids,
xmin=2, xmax=10,
xtick={3, 5, 7, 9},
y grid style={white},
ylabel={accuracy (\%)},
ymajorgrids,
ymin=0, ymax=3.0
]
\addlegendimage{no markers, color0, dashed}
\addlegendimage{no markers, color0}
\addlegendimage{no markers, color1, dashed}
\addlegendimage{no markers, color1}
\addplot [thick, color0, dashed]
table [row sep=\\]{%
3	0.5 \\
5	0.5 \\
7	0.5 \\
9	0.5 \\
};
\addplot [semithick, color0, mark=*, mark size=3, mark options={solid}]
table [row sep=\\]{%
3	2.6 \\
5	2.56 \\
7	2.26 \\
9	2.08 \\
};
\addplot [thick, color1, dashed]
table [row sep=\\]{%
3	0.1 \\
5	0.1 \\
7	0.1 \\
9	0.1 \\
};
\addplot [semithick, color1, mark=*, mark size=3, mark options={solid}]
table [row sep=\\]{%
3	0.63 \\
5	0.62 \\
7	0.55 \\
9	0.49 \\
};
\end{axis}

\end{tikzpicture}
\end{center}
\caption{Training a ReLU neural network, with two hidden-layers of size 512, on a single patch of size $k \times k$ from the ImageNet data. The patch is randomly chosen from inside the image. We train the networks with Adam, with batch size of $50$, for $10$k iterations.} \label{fig:imagenet}
\end{figure}
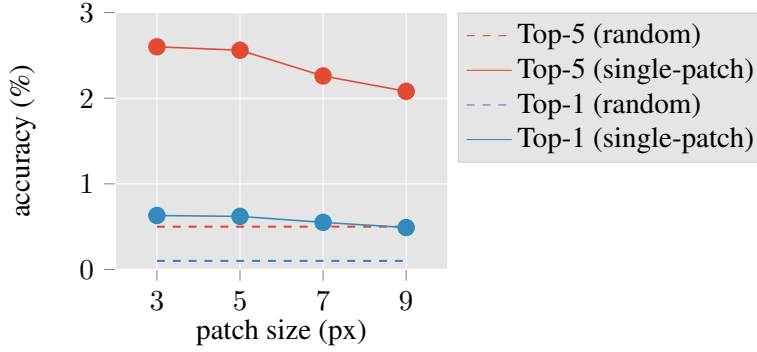

\section{Proof of \thmref{thm:generative_convergence_exact}}
To prove \thmref{thm:generative_convergence_exact}, we observe the behavior of the algorithm on the $i$-th layer.
Let $\psi : \{\pm 1\}^{n_i} \to \{\pm 1\}^{n_i}$ be some mapping such that
$\psi(\vx) = (\xi_1 \cdot x_1, \dots, \xi_{n_i} \cdot x_{n_i})$ for
$\xi_1, \dots, \xi_{n_i} \in \{\pm 1\}$.
We also define $\varphi_i : \{\pm 1\}^{n_i/2} \to \{\pm 1\}^{n_i/2}$ such that:
\[
\varphi_i(\vz) = (\nu_1 z_1, \dots, \nu_{n_i/2} z_{n_i/2})
\]
where $\nu_j := \begin{cases}
\sign(c_{i-1, j}) & c_{i-1,j} \ne 0 \\
1 & \mathcal{I}_{i-1,j} = 0
\end{cases}$

We can ignore examples that appear with probability zero. For this, we define the support of $\mathcal{D}$ by $\mathcal{X}' = \{\vx' \in \mathcal{X} ~:~ \prob{(\vx, y) \sim \mathcal{D}}{\vx = \vx'} > 0\}$.

We have the following important result, which we prove in the sequel:
\begin{lemma}
\label{lem:single_block}
Assume we initialize $\vw_l^{(0)}$ such that $\norm{\vw_l^{(0)}} \le \frac{1}{4k}$.
Fix $\delta > 0$. Assume we sample $S \sim \mathcal{D}$, with $\abs{S} > \frac{2^{11}}{\epsilon^2 \Delta^2} \log(\frac{8n_i}{\delta})$. 
Assume that $k \ge \log^{-1}(\frac{4}{3}) \log(\frac{8n_i}{\delta})$,
and that $\eta \le \frac{n_i}{16\sqrt{2}k}$.
Let $\Psi : \mathcal{X} \to [-1,1]^{n_i/2}$
such that for every $\vx \in \mathcal{X}'$
we have $\Psi(\vx) = \psi \circ \Gamma_{(i+1) \dots d}(\vx)$ for some $\psi$ as defined above.
Assume we perform the following updates:
\[
\mW^{(i)}_t \leftarrow \mW^{(i)}_{t-1} - \eta \frac{\partial}{\partial \mW^{(i)}_{t-1}} L_{\Psi(S)}(P(B_{\mW^{(i)}_{t-1}, \mV^{(i)}_0}))
\]
Then with probability at least $1-\delta$,
for $t > \frac{6n_i}{ \sqrt{2} \eta \epsilon \Delta}$ we have:
$B_{\mW^{(i)}_{t}, \mV^{(i)}_0}(\vx) = \varphi_i \circ \Gamma_i \circ \psi(\vx)$ for every $\vx \in \mathcal{X}'$.
\end{lemma}

Given this result, we can prove the main theorem:
\begin{proof} of \thmref{thm:generative_convergence_exact}.
Fix $\delta' = \frac{\delta}{d}$.
We show that for every $i \in [d]$, w.p at least $1-(d-i+1) \delta'$,
after the $i$-th step of the algorithm we have
$\mathcal{N}_{i-1}(\vx) = \varphi_i \circ \Gamma_{i \dots d} (\vx)$ for every $\vx \in \mathcal{X}'$.
By induction on $i$:
\begin{itemize}
\item For $i = d$, we get the required using \lemref{lem:single_block} with $\psi, \Psi = id$ and $\epsilon = \epsilon_{d-1}$, $\epsilon' = \epsilon_d$.
\item Assume the above holds for $i$, and we show it for $i-1$.
By the assumption, w.p at least $1-(d-i+1) \delta'$
we have $\mathcal{N}_{i-1}(\vx) = \varphi_i \circ \Gamma_{i \dots d}(\vx)$ for every $\vx \in \mathcal{X}'$.
Observe that:
\begin{align*}
\frac{\partial L_{\mathcal{D}}}{\partial \mW^{(i-1)}_{t}} (P(B_{\mW^{(i-1)}_{t-1}, \mV^{(i-1)}_0} \circ \mathcal{N}_{i-1}))
&= \frac{\partial L_{\mathcal{N}_{i-1}(\mathcal{D})}}{\partial \mW^{(i-1)}_{t}} (P(B_{\mW^{(i-1)}_{t}, \mV^{(i-1)}_0}))
\end{align*}
So using \lemref{lem:single_block} with $\psi = \varphi_i$,
$\Psi = \mathcal{N}_{i-1}$ we get that w.p at least $1-\delta'$ we have
$B_{\mW^{(i-1)}_{T}, \mV^{(i-1)}_0}(\vx) = \varphi_{i-1} \circ \Gamma_{i-1} \circ \varphi_i(\vx)$ for every $\vx \in \mathcal{X}'$.
In this case, since $\varphi_i \circ \varphi_i = id$, we get that
for every $\vx \in \mathcal{X}'$:
\begin{align*}
\mathcal{N}_{i-2}(\vx) &= B_{\mW^{(i-1)}_{T}, \mV^{(i-1)}_0} \circ \mathcal{N}_{i-1}(\vx) \\
&= (\varphi_{i-1} \circ \Gamma_{i-1} \circ \varphi_i) \circ (\varphi_i \circ \Gamma_{i \dots d})(\vx) = \varphi_{i-1} \circ \Gamma_{(i-1) \dots d}(\vx)
\end{align*}
and using the union bound gives the required.
\end{itemize}

Notice that $\varphi_1 = id$:
by definition of $\D{0} = \Gamma_{1 \dots d}(\mathcal{D})$, for $(\vz,y) \sim \D{0}$ we have $\vz = \Gamma_{1 \dots d}(\vx)$ and also $y = \Gamma_{1 \dots d}(\vx)$ for $(\vx, y) \sim \mathcal{D}$. Therefore, we have $c_{0,1} = \mean{(x,y) \sim \D{0}}{xy} = 1$,
and therefore $\varphi_i(z) = \sign(c_{0,1}) z = z$.
Now, choosing $i = 1$, the above result shows that with probability at least $1-\delta$, the algorithm returns
$\mathcal{N}_0$ such that $\mathcal{N}_0(\vx) = \varphi_1 \circ \Gamma_1 \circ \dots \circ \Gamma_d(\vx) =
h_C(\vx)$ for every $\vx \in \mathcal{X}'$.
\end{proof}

In the rest of this section we prove \lemref{lem:single_block}.
Fix some $i \in [d]$ and let $j \in [n_i/2]$.
With slight abuse of notation, we denote by $\vw^{(t)}$ the value of the weight
$\vw^{(i,j)}$ at iteration $t$, and denote $\vv := \vv^{(i,j)}$ and
$g_t := g_{\vw^{(t)}, \vv}$.
Recall that we defined $\psi(\vx) = (\xi_1 \cdot x_1, \dots, \xi_{n_i} \cdot x_{n_i})$ for $\xi_1 \dots \xi_{n_i} \in \{\pm 1\}$.
Let $\gamma := \gamma_{i-1,j}$,
and let $\widetilde{\gamma}$ such that $\widetilde{\gamma}(x_1, x_2) = \gamma(\xi_{2j-1} \cdot x_1, \xi_{2j} \cdot x_2)$. For every $\vp \in \{\pm 1\}^2$, denote $\widetilde{\vp} := (\xi_{2j-1} p_1, \xi_{2j} p_2)$, so we have $\gamma(\widetilde{\vp}) = \widetilde{\gamma}(\vp)$.
Now, we care only about patterns $\vp$ that have positive probability to appear as input to the gate $(i-1,j)$. So, we define our pattern support by:
\[
\mathcal{P} = \{\vp \in \{\pm 1\}^2 ~:~ \prob{(\vx, y) \sim \Dt{i}}{(x_{2j-1}, x_{2j}) = \vp} > 0 \}
\]
We start by observing the behavior of the gradient with respect to some pattern $\vp \in \mathcal{P}$:

\begin{lemma}
\label{lem:bound_orig_distribution}
Fix some $\vp \in \mathcal{P}$. For every $l \in [k]$ such that $\inner{\vw_l^{(t)}, \vp } > 0$ and $g_t(\vp) \in (-1,1)$, the following holds:
\begin{align*}
-\widetilde{\gamma}(\vp) v_l \nu_j \inner{\frac{\partial L_{\Dt{i}}}{\partial \vw_l^{(t)}}, \vp}>\frac{\epsilon }{\sqrt{2}n_i} \Delta
\end{align*}
\end{lemma}

\begin{proof} Observe the following:
\begin{align*}
&\frac{\partial L_{\Dt{i}}}{\partial \vw_l^{(t)}} 
(P(B_{\mW^{(i)}, \mV^{(i)}})) \\
&= \mean{(\vx,y) \sim \Dt{i}}{\ell'(P(B_{\mW^{(i)}, \mV^{(i)}})(\vx))
\cdot \frac{\partial}{\partial \vw_l^{(t)}} \frac{2}{n_i} \sum_{j'=1}^{n_i/2} g_{\vw^{(i,j')}, \vv^{(i,j')}}(x_{2j'-1}, x_{2j'})}\\
&+ \mean{(\vx,y) \sim \Dt{i}}{R_\lambda'(P(B_{\mW^{(i)}, \mV^{(i)}})(\vx))
\cdot \frac{\partial}{\partial \vw_l^{(t)}} \frac{2}{n_i} \sum_{j'=1}^{n_i/2} g_{\vw^{(i,j')}, \vv^{(i,j')}}(x_{2j'-1}, x_{2j'})}\\
&= \frac{2}{n_i} \mean{\Dt{i}}{(\lambda-y)
\frac{\partial}{\partial \vw_l^{(t)}} g_{t}(x_{2j-1}, x_{2j})}\\
&= \frac{2}{n_i} \mean{\Dt{i}}{(\lambda-y)v_l
\1\{g_{t}(x_{2j-1}, x_{2j}) \in (-1,1)\} \cdot \1\{\inner{\vw_l^{(t)}, (x_{2j-1}, x_{2j})} > 0\} \cdot (x_{2j-1}, x_{2j})}\\
\end{align*}
We use the fact that $\ell'(P(B_{\mW^{(i)}, \mV^{(i)}})(\vx)) = -y$, unless $P(B_{\mW^{(i)}, \mV^{(i)}})(\vx) \in \{\pm 1\}$, in which case $g_{t}(x_{2j-1}, x_{2j}) \in \{\pm 1\}$, so $\frac{\partial}{\partial \vw_l^{(t)}} g_{t}(x_{2j-1}, x_{2j}) = 0$. Similarly, unless $\frac{\partial}{\partial \vw_l^{(t)}} g_{t}(x_{2j-1}, x_{2j}) = 0$, we get that $R_\lambda'(P(B_{\mW^{(i)}, \mV^{(i)}})(\vx))=\lambda$.
Fix some $\vp \in \{\pm 1\}^2$ such that $\inner{\vw_l^{(t)}, \vp} > 0$.
Note that for every $\vp \ne \vp' \in \{\pm 1\}^2$ we have either $\inner{\vp, \vp'} = 0$,
or $\vp = -\vp'$ in which case $\inner{\vw_l^{(t)}, \vp'} < 0$.
Therefore, we get the following:
\begin{align*}
&\inner{\frac{\partial L_{\Dt{i}}}{\partial \vw_l^{(t)}}, \vp} \\
&= \frac{2}{n_i} \mean{\Dt{i}}{(\lambda-y)v_l
\1\{g_{t}(x_{2j-1}, x_{2j}) \in (-1,1)\} \cdot \1\{\inner{\vw_l^{(t)}, (x_{2j-1}, x_{2j})} \ge 0\} \cdot \inner{(x_{2j-1}, x_{2j}), \vp}} \\
&= \frac{2}{n_i} \mean{\Dt{i}}{(\lambda-y)v_l
\1\{g_{t}(x_{2j-1}, x_{2j}) \in (-1,1)\} \cdot \1\{(x_{2j-1}, x_{2j}) = \vp\} \norm{\vp}} \\
\end{align*}

Denote $q_\vp := \prob{(\vx,y) \sim \D{i}}{(x_{2j-1},x_{2j}) = \vp| 
\gamma(x_{2j-1},x_{2j}) = \gamma(\vp)}$.
Using property \ref{prp:cond_indep}, we have:
\begin{align*}
&\prob{(\vx,y) \sim \D{i}}{(x_{2j-1},x_{2j}) = \vp, y = y'} \\
&= \prob{(\vx,y) \sim \D{i}}{(x_{2j-1},x_{2j}) = \vp, y = y',
\gamma(x_{2j-1},x_{2j}) = \gamma(\vp)} \\
&=\prob{(\vx,y) \sim \D{i}}{(x_{2j-1},x_{2j}) = \vp, y = y' | 
\gamma(x_{2j-1},x_{2j}) = \gamma(\vp)} \prob{(\vx,y) \sim \D{i}}
{\gamma(x_{2j-1},x_{2j}) = \gamma(\vp)} \\
&= q_\vp \prob{(\vx,y) \sim \D{i}}{ 
\gamma(x_{2j-1},x_{2j}) = \gamma(\vp),y = y'} \\
&=q_\vp \prob{(\vz,y) \sim \D{i-1}}{ 
z_j = \gamma(\vp),y = y'}
\end{align*}
And therefore:
\begin{align*}
\mean{(\vx,y) \sim \D{i}}{y\1\{(x_{2j-1},x_{2j}) = \vp\}}
&= \sum_{y' \in \{\pm 1\}} y'
\prob{(\vx,y) \sim \D{i}}{(x_{2j-1},x_{2j}) = \vp, y = y'} \\
&= q_\vp \sum_{y' \in \{\pm 1\}} y'
\prob{(\vz,y) \sim \D{i-1}}{z_j = \gamma(\vp), y = y'} \\
&= q_\vp \mean{(\vz,y) \sim \D{i-1}}{y\1\{z_j = \gamma(\vp)\}}
\end{align*}
Assuming $g_{t}(\vp) \in (-1,1)$, using the above we get:
\begin{align*}
\inner{\frac{\partial L_{\Dt{i}}}{\partial \vw_l^{(t)}}, \vp}
&= \frac{2 \sqrt{2} v_l}{n_i}\mean{(\vx,y) \sim \Dt{i}}{(\lambda-y)\1\{(x_{2j-1}, x_{2j}) = \vp\}} \\
&= \frac{2 \sqrt{2} v_l}{n_i}\mean{(\vx,y) \sim \D{i}}{(\lambda-y)\1\{(\xi_{2j-1}x_{2j-1},
\xi_{2j} x_{2j}) = \vp\}} \\
&= \frac{2 \sqrt{2} v_l}{n_i}\mean{(\vx,y) \sim \D{i}}{(\lambda-y)\1\{(x_{2j-1},
x_{2j}) = \widetilde{\vp}\}} \\
&= \frac{2 \sqrt{2}v_l q_{\widetilde{\vp}}}{n_i} \mean{(\vz,y) \sim \D{i-1}}{(\lambda-y)
\1\{z_j = \widetilde{\gamma}(\vp)\}} \\
\end{align*}

Now, we have the following cases:
\begin{itemize}
\item If $\mathcal{I}_{i-1,j} = 0$, then by property \ref{prp:correlation} $z_j$ and $y$ are independent, so:
\begin{align*}
\inner{\frac{\partial L_{\Dt{i}}}{\partial \vw_l^{(t)}}, \vp}
&= \frac{2 \sqrt{2}v_l q_{\widetilde{\vp}}}{n_i} \mean{(\vz,y) \sim \D{i-1}}{(\lambda-y)
\1\{z_j = \widetilde{\gamma}(\vp)\}} \\
&= \frac{2 \sqrt{2}v_l q_{\widetilde{\vp}}}{n_i} \mean{(\vz,y) \sim \D{i-1}}{(\lambda-y)
}\prob{(\vz,y) \sim \D{i-1}}{z_j = \widetilde{\gamma}(\vp)} \\
&= \frac{2 \sqrt{2}v_l}{n_i} (\lambda-\mean{(\vz,y) \sim \D{i-1}}{y
})\prob{(\vx,y) \sim \D{i}}{(x_{2j-1}, x_{2j}) = \widetilde{\vp}} \\
\end{align*}
Since we assume $\widetilde{\gamma}(\vp) = 1, \nu_j = 1$, and using property \ref{prp:cond_bound} and the fact that $\vp \in \mathcal{P}$, we get that:
\begin{align*}
-\widetilde{\gamma}(\vp) v_l \nu_j \inner{\frac{\partial L_{\Dt{i}}}{\partial \vw_l^{(t)}}, \vp} &= - v_l  \inner{\frac{\partial L_{\Dt{i}}}{\partial \vw_l^{(t)}}, \vp} \\
&= \frac{2 \sqrt{2}}{n_i} (\lambda - \mean{}{y}) \prob{(\vx,y) \sim \D{i}}{(x_{2j-1}, x_{2j}) = \widetilde{\vp}} > \frac{\Delta \epsilon}{\sqrt{2}n_i}
\end{align*}

Using the fact that $\lambda = \mean{}{y} + \frac{\Delta}{4}$.

\item Otherwise, observe that:
\begin{align*}
\inner{\frac{\partial L_{\Dt{i}}}{\partial \vw_l^{(t)}}, \vp}
&= \frac{2 \sqrt{2}v_l q_{\widetilde{\vp}}}{n_i} \mean{(\vz,y) \sim \D{i-1}}{(\lambda-y)
\1\{z_j = \widetilde{\gamma}(\vp)\}} \\
&= \frac{2\sqrt{2}v_l q_{\widetilde{\vp}}}{n_i} \left(\lambda\prob{(\vz,y) \sim \D{i-1}}{z_j = \widetilde{\gamma}(\vp)} -\mean{(\vz,y) \sim \D{i-1}}{y
\frac{1}{2}(z_j \cdot \widetilde{\gamma}(\vp) + 1)}\right) \\
&= \frac{\sqrt{2}v_l q_{\widetilde{\vp}}}{n_i} \left( 2\lambda\prob{(\vz,y) \sim \D{i-1}}{z_j = \widetilde{\gamma}(\vp)} - \widetilde{\gamma}(\vp)
c_{i-1,j} - \mean{(\vz,y) \sim \D{i-1}}{y} \right)\\
\end{align*}
And therefore we get:
\begin{align*}
-\widetilde{\gamma}(\vp) v_l \sign(c_{i-1,j}) \inner{\frac{\partial L_{\Dt{i}}}{\partial \vw_l^{(t)}}, \vp}
&= \frac{\sqrt{2}q_{\widetilde{\vp}} }{n_i}
\left(\abs{c_{i-1,j}} +  \sign(c_{i-1,j})\widetilde{\gamma}(\vp) (\mean{}{y} - 2\lambda\prob{}{z_j = \widetilde{\gamma}(\vp)}) \right) \\
\end{align*}
Now, if $\sign(c_{i-1,j}) \widetilde{\gamma}(\vp) = 1$, using property \ref{prp:correlation}, since $\mathcal{I}_{i-1,j} \ne 0$ we get:
\begin{align*}
-\widetilde{\gamma}(\vp) v_l \sign(c_{i-1,j}) \inner{\frac{\partial L_{\Dt{i}}}{\partial \vw_l^{(t)}}, \vp} \ge \frac{\sqrt{2}q_{\widetilde{\vp}} }{n_i}
\left(\abs{c_{i-1,j}} + \mean{}{y} - 2\lambda \right) >
\frac{\epsilon}{\sqrt{2}n_i} \Delta
\end{align*}
Otherwise, we have $\sign(c_{i-1,j}) \widetilde{\gamma}(\vp) = -1$, and then:
\begin{align*}
-\widetilde{\gamma}(\vp) v_l \sign(c_{i-1,j}) \inner{\frac{\partial L_{\Dt{i}}}{\partial \vw_l^{(t)}}, \vp} \ge \frac{\sqrt{2}q_{\widetilde{\vp}} }{n_i}
\left(\abs{c_{i-1,j}} - \mean{}{y} \right) >
\frac{\sqrt{2}\epsilon}{n_i} \Delta
\end{align*}
where we use property \ref{prp:cond_bound} and the fact that $\vp \in \mathcal{P}$.
\end{itemize}
\end{proof}

We introduce the following notation: for a sample $S \subseteq \mathcal{X}' \times \mathcal{Y}$, and some function $f : \mathcal{X}' \to \mathcal{X}'$, denote by $f(S)$ the sample $f(S) := \{(f(\vx), y)\}_{(\vx, y) \in S}$. Using standard concentration of measure arguments, we get that the gradient on the sample approximates the gradient on the distribution:
\begin{lemma}
\label{lem:noisy_sample}
Fix $\delta > 0$. Assume we sample $S \sim \mathcal{D}$, with $\abs{S} > \frac{2^{11}}{\epsilon^2 \Delta^2} \log\frac{8}{\delta}$. Then,  with probability at least $1-\delta$, for every $\vp \in \{\pm 1\}^2$ such that $\inner{\vw_l^{(t)}, \vp} >0$ it holds that:
\[
\abs{\inner{\frac{\partial L_{\Psi(\mathcal{D})}}{\partial \vw_l^{(t)}}, \vp} - \inner{\frac{\partial L_{\Psi(S)}}{\partial \vw_l^{(t)}}, \vp}} \le \frac{\epsilon}{4\sqrt{2} n_i} \Delta
\]
\end{lemma}
\begin{proof}
Fix some $\vp \in \{\pm 1\}^2$ with $\inner{\vw_l^{(t)}, \vp} >0$.
Similar to what we previously showed, we get that:
\begin{align*}
&\inner{\frac{\partial L_{\Psi(S)}}{\partial \vw_l^{(t)}}, \vp} \\
&= \frac{2}{n_i} \mean{(\vx,y) \sim \Psi(S)}{(\lambda-y)v_l
\1\{g_{t}(x_{2j-1}, x_{2j}) \in (-1,1)\} \cdot \1\{\inner{\vw_l^{(t)}, (x_{2j-1}, x_{2j})} \ge 0\} \cdot \inner{(x_{2j-1}, x_{2j}), \vp}} \\
&= \frac{2}{n_i} \mean{(\vx,y) \sim \Psi(S)}{(\lambda-y)v_l
\1\{g_{t}(x_{2j-1}, x_{2j}) \in (-1,1)\} \cdot \1\{(x_{2j-1}, x_{2j}) = \vp\} \norm{\vp}} \\
&= \frac{2\sqrt{2}}{n_i} \mean{(\vx,y) \sim \Psi(S)}{(\lambda-y)v_l
\1\{g_{t}(x_{2j-1}, x_{2j}) \in (-1,1)\} \cdot \1\{(x_{2j-1}, x_{2j}) = \vp\}} \\
\end{align*}
Denote $f(\vx,y) = (\lambda-y)v_l
\1\{g_{t}(x_{2j-1}, x_{2j}) \in (-1,1)\} \cdot \1\{(x_{2j-1}, x_{2j}) = \vp\}$,
and notice that since $\lambda \le 1$, we have $f(\vx,y) \in [-2,2]$.
Now, from Hoeffding's inequality we get that:
\[
\prob{S}{\abs{\mean{\Psi(S)}{f(\vx,y)}-\mean{\Psi(\mathcal{D})}{f(\vx,y)}} \ge \tau}
\le 2\exp \left(- \frac{1}{8} |S| \tau^2 \right)
\]
So, for $|S| > \frac{8}{\tau^2} \log\frac{8}{\delta}$ we get that with probability at least $1-\frac{\delta}{4}$ we have:
\begin{align*}
\abs{\inner{\frac{\partial L_{\Psi(\mathcal{D})}}{\partial \vw_l^{(t)}}, \vp} - \inner{\frac{\partial L_{\Psi(S)}}{\partial \vw_l^{(t)}}, \vp}}
&= \frac{2\sqrt{2}}{n_i} \abs{\mean{\Psi(S)}{f(\vx,y)}-\mean{\Psi(\mathcal{D})}{f(\vx,y)}} < \frac{2 \sqrt{2}}{n_i}\tau
\end{align*}
Taking $\tau = \frac{\epsilon}{16}\Delta$ and using the union bound over all $\vp \in \{\pm 1\}^2$ completes the proof.
\end{proof}

Using the two previous lemmas, we can estimate the behavior of the gradient on the sample, with respect to a given pattern $\vp$:
\begin{lemma}
\label{lem:bound_noisy_distribution}
Fix $\delta > 0$. Assume we sample $S \sim \mathcal{D}$, with $\abs{S} > \frac{2^{11}}{\epsilon^2 \Delta^2} \log\frac{8}{\delta}$. Then,  with probability at least $1-\delta$, for every $\vp \in \mathcal{P}$, and for every $l \in [k]$ such that $\inner{\vw_l^{(t)}, \vp } > 0$ and $g_t(\vp) \in (-1,1)$, the following holds:
\begin{align*}
-\widetilde{\gamma}(\vp) v_l \nu_j \inner{\frac{\partial L_{\Psi(S)}}{\partial \vw_l^{(t)}}, \vp} > \frac{\epsilon}{2\sqrt{2}n_i} \Delta
\end{align*}
\end{lemma}

\begin{proof}
Using \lemref{lem:bound_orig_distribution} and \lemref{lem:noisy_sample}, with probability at least $1-\delta$:
\begin{align*}
-\widetilde{\gamma}(\vp) v_l \nu_j \inner{\frac{\partial L_{\Psi(S)}}{\partial \vw_l^{(t)}}, \vp} 
&= -\widetilde{\gamma}(\vp) v_l \nu_j \left(\inner{\frac{\partial L_{\Dt{i}}}{\partial \vw_l^{(t)}}, \vp} + \inner{\frac{\partial L_{\Psi(S)}}{\partial \vw_l^{(t)}}, \vp} - \inner{\frac{\partial L_{\Psi(\mathcal{D})}}{\partial \vw_l^{(t)}}, \vp} \right) \\
&\ge -\widetilde{\gamma}(\vp) v_l \nu_j \inner{\frac{\partial L_{\Dt{i}}}{\partial \vw_l^{(t)}}, \vp}
- \abs{\inner{\frac{\partial L_{\Psi(S)}}{\partial \vw_l^{(t)}}, \vp}- \inner{\frac{\partial L_{\Psi(\mathcal{D})}}{\partial \vw_l^{(t)}}, \vp}} \\
&> \frac{\epsilon}{\sqrt{2}n_i} \Delta
-\frac{\epsilon}{4 \sqrt{2}n_i} \Delta \ge \frac{3 \epsilon}{4\sqrt{2}n_i} \Delta
\end{align*}
\end{proof}

We want to show that if the value of $g_t$ gets ``stuck'', then it recovered the
value of the gate, multiplied by the correlation $c_{i-1,j}$.
We do this by observing the dynamics of $\inner{\vw_l^{(t)}, \vp}$.
In most cases, its value moves in the right direction, except for a small
set that oscillates around zero. This set is the following:
\begin{align*}
A_t =& \left\lbrace (l, \vp) ~:~ \vp \in \mathcal{P} \wedge \widetilde{\gamma}(\vp) v_l \nu_j < 0 \wedge
\inner{\vw_l^{(t)}, \vp} \le \frac{4\sqrt{2}\eta}{n_i}
\wedge \left(
\widetilde{\gamma}(-\vp) v_l \nu_j < 0 \vee - \vp \in \mathcal{P} \right)  \right\rbrace 
\end{align*}

We have the following simple observation:
\begin{lemma} With the assumptions of \lemref{lem:bound_noisy_distribution}, with probability at least $1-\delta$, for every $t$ we have: $A_t \subseteq A_{t+1}$.
\end{lemma}
\begin{proof}
Fix some $(l,\vp) \in A_t$, and we need to show that
$\inner{\vw_l^{(t+1)}, \vp} \le \frac{4\sqrt{2}\eta}{n_i}$.
If $\inner{\vw_l^{(t)}, \vp} = 0$ then\footnote{We take the sub-gradient zero at zero.}
$\inner{\vw_l^{(t+1)}, \vp}= \inner{\vw_l^{(t)}, \vp} \le \frac{4\sqrt{2}\eta}{n_i}$ and we are done.
If $\inner{\vw_l^{(t)}, \vp} > 0$ then, since $\vp \in \mathcal{P}$ we have from \lemref{lem:bound_noisy_distribution}, w.p at least $1-\delta$:
\begin{align*}
- \inner{\frac{\partial L_{\Psi(S)}}{\partial \vw_l^{(t)}}, \vp} < \widetilde{\gamma}(\vp) v_l \nu_j \frac{\epsilon}{2\sqrt{2}n_i} \Delta < 0
\end{align*}
Where we use the fact that $\widetilde{\gamma}(\vp) v_l \nu_j < 0$.
Therefore, we get:
\[
\inner{\vw_l^{(t+1)}, \vp}
= \inner{\vw_l^{(t)}, \vp} - \eta
\inner{\frac{\partial L_{\Psi(S)}}{\partial \vw_l^{(t)}}, \vp}
\le \inner{\vw_l^{(t)}, \vp}
\le \frac{4\sqrt{2}\eta}{n_i}
\]
Otherwise, we have $\inner{\vw_l^{(t)}, \vp} < 0$, so:
\begin{align*}
\inner{\vw_l^{(t+1)}, \vp}
&= \inner{\vw_l^{(t)}, \vp} - \eta
\inner{\frac{\partial L_{\Psi(S)}}{\partial \vw_l^{(t)}}, \vp} 
\le \inner{\vw_l^{(t)}, \vp} + \frac{4\sqrt{2} \eta}{n_i}
\le \frac{4 \sqrt{2}\eta}{n_i}
\end{align*}
\end{proof}

Now, we want to show that all $\inner{\vw_l^{(t)}, \vp}$ with $(l,\vp) \notin A_t$ and $\vp \in \mathcal{P}$
move in the direction of $\widetilde{\gamma}(\vp) \cdot \nu_j$:
\begin{lemma}
With the assumptions of \lemref{lem:bound_noisy_distribution}, with probability at least $1-\delta$,
for every $l$,$t$ and $\vp \in \mathcal{P}$ such that
$\inner{\vw_l^{(t)}, \vp} > 0$ and $(l, \vp) \notin A_t$, it holds that:
\[
\left(\sigma(\inner{\vw_l^{(t)},\vp}) - \sigma(\inner{\vw_l^{(t-1)},\vp})\right) \cdot
\widetilde{\gamma}(\vp) v_l \nu_j \ge 0
\]
\end{lemma}
\begin{proof}
Assume the result of \lemref{lem:bound_noisy_distribution} holds (this happens with probability at least $1-\delta$).
We cannot have $\inner{\vw_l^{(t-1)}, \vp} = 0$, since otherwise we would have
$\inner{\vw_l^{(t)}, \vp} = 0$, contradicting the assumption.
If $\inner{\vw_l^{(t-1)}, \vp} > 0$, since we require
$\inner{\vw_l^{(t)}, \vp} > 0$ we get that:
\[
\sigma(\inner{\vw_l^{(t)},\vp}) - \sigma(\inner{\vw_l^{(t-1)},\vp})
= \inner{\vw_l^{(t)}-\vw_l^{(t-1)}, \vp} = -\eta
\inner{\frac{\partial L_{\Psi(S)}}{\partial \vw_l^{(t-1)}}, \vp}
\]
and the required follows from \lemref{lem:bound_noisy_distribution}.
%
Otherwise, we have $\inner{\vw_l^{(t-1)}, \vp} < 0$.
We observe the following cases:

\begin{itemize}
\item If $\widetilde{\gamma}(\vp) v_l \nu_j \ge 0$ then we are done, since:
\[\left(\sigma(\inner{\vw_l^{(t)},\vp}) - \sigma(\inner{\vw_l^{(t-1)},\vp})\right) \cdot
\widetilde{\gamma}(\vp) v_l\nu_j
= \sigma(\inner{\vw_l^{(t)},\vp}) \cdot \widetilde{\gamma}(\vp) v_l \nu_j \ge 0
\]
\item Otherwise, we have $\widetilde{\gamma}(\vp) v_l \nu_j < 0$. We also have:
\begin{align*}
\inner{\vw_l^{(t)}, \vp}
&= \inner{\vw_l^{(t-1)}, \vp} - \eta
\inner{\frac{\partial L_{\Psi(S)}}{\partial \vw_l^{(t)}}, \vp} 
\le \inner{\vw_l^{(t-1)}, \vp} + \frac{4\sqrt{2} \eta}{n_i}
\le \frac{4\sqrt{2} \eta}{n_i}
\end{align*}
Since we assume $(l, \vp) \notin A_t$, we must have $-\vp \in \mathcal{P}$ and $\widetilde{\gamma}(-\vp) v_l \nu_j \ge 0$. Therefore, from \lemref{lem:bound_noisy_distribution} we get:
\begin{align*}
\inner{\frac{\partial L_{\Psi(S)}}{\partial \vw_l^{(t)}}, -\vp} < -\widetilde{\gamma}(-\vp) v_l \nu_j\frac{\epsilon}{2\sqrt{2}n_i} \Delta
\end{align*}
And hence:
\begin{align*}
0 &< \inner{\vw_l^{(t)}, \vp}
= \inner{\vw_l^{(t-1)}, \vp} + \eta\inner{\frac{\partial L_{\Psi(S)}}{\partial \vw_l^{(t-1)}}, -\vp}
\le -\eta \widetilde{\gamma}(-\vp) v_l \nu_j\frac{\epsilon}{2\sqrt{2}n_i} \Delta < 0
\end{align*}
and we reach a contradiction.
\end{itemize}
\end{proof}

From the above, we get the following:
\begin{corollary}
\label{crl:neutral_neuron}
With the assumptions of \lemref{lem:bound_noisy_distribution}, with probability at least $1-\delta$, for every $l$,$t$ and $\vp \in \mathcal{P}$ such that
$\inner{\vw_l^{(t)}, \vp} > 0$ and $(l, \vp) \notin A_t$, the following holds:
\[
\left(\sigma(\inner{\vw_l^{(t)},\vp}) - \sigma(\inner{\vw_l^{(0)},\vp})\right) \cdot
\widetilde{\gamma}(\vp) v_l \nu_j \ge 0
\]
\end{corollary}
\begin{proof}
Notice that for every $t' \le t$ we have $(l,\vp) \notin A_{t'} \subseteq A_t$.
Therefore, using the previous lemma:
\begin{align*}
\left(\sigma(\inner{\vw_l^{(t)},\vp}) - \sigma(\inner{\vw_l^{(0)},\vp})\right) \cdot
\widetilde{\gamma}(\vp) v_l \nu_j
&= \sum_{1 \le t' \le t} \left(\sigma(\inner{\vw_l^{(t)},\vp}) - \sigma(\inner{\vw_l^{(t')},\vp})\right) \cdot
\widetilde{\gamma}(\vp) v_l \nu_j \ge 0
\end{align*}
\end{proof}

Finally, we need to show that there are some ``good'' neurons, that are moving
strictly away from zero:
\begin{lemma}
\label{lem:good_neuron}
Fix $\delta > 0$. Assume we sample $S \sim \mathcal{D}$, with $\abs{S} > \frac{2^{11}}{\epsilon^2 \Delta^2} \log\frac{8}{\delta}$. 
Assume that $k \ge \log^{-1}(\frac{4}{3}) \log(\frac{4}{\delta})$.
Then with probability at least $1-2\delta$, for every $\vp \in \mathcal{P}$,
there exists $l \in [k]$ such that
for every $t$ with $g_{t-1}(\vp) \in (-1,1)$, we have:
\[
\sigma(\inner{\vw_l^{(t)}, \vp}) \cdot \widetilde{\gamma}(\vp) v_l \nu_j \ge 
\eta t \frac{\epsilon}{2\sqrt{2}n_i} \Delta
\]
\end{lemma}
\begin{proof}
Assume the result of \lemref{lem:bound_noisy_distribution} holds (happens with probability at least $1-\delta$).
Fix some $\vp \in \mathcal{P}$.
For $l \in [k]$, with probability $\frac{1}{4}$ we have both $v_l = \widetilde{\gamma}(\vp) \nu_j$ and $\inner{\vw_l^{(0)}, \vp} > 0$.
Therefore, the probability that there exists $l \in [k]$ such that
the above holds is $1-(\frac{3}{4})^k \ge 1-\frac{\delta}{4}$.
Using the union bound, w.p at least $1-\delta$, there exists such $l \in [k]$
for every $\vp \in \{\pm 1\}^2$.
In such case, we have $\inner{\vw_l^{(t)}, \vp}
\ge \eta t \frac{\epsilon}{2\sqrt{2}n_i} \Delta$, by induction:
\begin{itemize}
\item For $t = 0$ this is true since $\inner{\vw_l^{(0)}, \vp} > 0$.
\item If the above holds for $t - 1$, then $\inner{\vw_l^{(t-1)}, \vp} > 0$,
and therefore, using $v_l = \widetilde{\gamma}(\vp) \nu_j$ and \lemref{lem:bound_noisy_distribution}:

\begin{align*}
- \inner{\frac{\partial L_{\Psi(\mathcal{D})}}{\partial \vw_l^{(t)}}, \vp} > \widetilde{\gamma}(\vp) v_l \nu_j \frac{\epsilon}{2\sqrt{2}n_i} \Delta
\end{align*}
And we get:
\begin{align*}
\inner{\vw_l^{(t)}, \vp} &= 
\inner{\vw_l^{(t-1)}, \vp} - \eta \inner{\frac{\partial L_{\Psi(\mathcal{D})}}{\partial \vw_l^{(t)}}, \vp} \\
&>
\inner{\vw_l^{(t-1)}, \vp} + \eta\widetilde{\gamma}(\vp) v_l \nu_j \frac{\epsilon}{2\sqrt{2}n_i}\Delta\\
&\ge
\eta (t-1) \frac{\epsilon}{2\sqrt{2}n_i} \Delta + \eta \frac{\epsilon}{2\sqrt{2}n_i} \Delta
\end{align*}
\end{itemize}
\end{proof}

Using the above results, we can analyze the behavior of $g_t(\vp)$:
\begin{lemma}
\label{lem:single_gate}
Assume we initialize $\vw_l^{(0)}$ such that $\norm{\vw_l^{(0)}} \le \frac{1}{4k}$.
Fix $\delta > 0$. Assume we sample $S \sim \mathcal{D}$, with $\abs{S} > \frac{2^{11}}{\epsilon^2 \Delta^2} \log\frac{8}{\delta}$.
Then with probability at least $1-2\delta$, for every $\vp \in \mathcal{P}$,
for $t > \frac{6n_i}{ \sqrt{2} \eta \epsilon\Delta}$ we have:
\[
g_t(\vp) = \widetilde{\gamma}(\vp) \nu_j
\]
\end{lemma}
\begin{proof}
Using \lemref{lem:good_neuron}, w.p at least $1-2\delta$, for every such $\vp$ there exists
$l_\vp \in [k]$ such that for every $t$ with $g_{t-1}(\vp) \in (-1,1)$:
\[
v_{l_\vp}\sigma(\inner{\vw_{l_\vp}^{(t)}, \vp}) \cdot \widetilde{\gamma}(\vp) \nu_j \ge \eta t \frac{\epsilon}{2\sqrt{2}n_i} \Delta
\]
Assume this holds, and fix some $\vp \in \mathcal{P}$.
Let $t$, such that $g_{t-1}(\vp) \in (-1,1)$.
Denote the set of indexes $J = \{l ~:~ \inner{\vw_l^{(t)}, \vp} > 0\}$.
We have the following:
\begin{align*}
g_t(\vp) &= \sum_{l \in J} v_l \sigma(\inner{\vw_l^{(t)}, \vp}) \\
&= v_{l_{\vp}}\sigma(\inner{\vw_{l_\vp}^{(t)}, \vp}) 
+ \sum_{l \in J \setminus \{l_\vp\}, (l, \vp) \notin A_t} v_l \sigma(\inner{\vw_l^{(t)}, \vp}) 
+ \sum_{l \in J \setminus \{l_\vp\}, (l, \vp) \in A_t} v_l \sigma(\inner{\vw_l^{(t)}, \vp})\\
\end{align*}
From \crlref{crl:neutral_neuron} we have:
\begin{align*}
\widetilde{\gamma}(\vp) \nu_j \cdot \sum_{l \in J \setminus \{l_\vp\}, (l, \vp) \notin A_t} v_l \sigma(\inner{\vw_l^{(t)}, \vp}) &\ge 
-k \sigma(\inner{\vw_l^{(0)}, \vp})
\ge -\frac{1}{4}
\end{align*}
By definition of $A_t$ and by our assumption on $\eta$ we have:
\begin{align*}
\widetilde{\gamma}(\vp) \nu_j \cdot \sum_{l \in J \setminus \{l_\vp\}, (l, \vp) \in A_t} v_l \sigma(\inner{\vw_l^{(t)}, \vp}) \ge -k \frac{4\sqrt{2} \eta}{n_i}
\ge - \frac{1}{4}
\end{align*}
Therefore, we get:
\[
\widetilde{\gamma}(\vp) \nu_j \cdot g_t(\vp)
\ge \eta t \frac{\epsilon}{2\sqrt{2}n_i} \Delta - \frac{1}{2}
\]
This shows that for $t > \frac{6n_i}{ \sqrt{2} \eta \epsilon \Delta}$ we get the required.
\end{proof}

\begin{proof} of \lemref{lem:single_block}.
Using the result of \lemref{lem:single_gate},
with union bound over all choices of $j \in [n_i/2]$. The required follows by the definition
of $\widetilde{\gamma}(x_{2j-1}, x_{2j}) = \gamma_{i-1,j}(\xi_{2j-1} x_{2j-1}, 
\xi_{2j} x_{2j})$.
\end{proof}

\section{Proofs of Section \ref{sec:main_results}}
\begin{proof}of \lemref{lem:prod_dist}.
Property \ref{prp:correlation} is immediate from assumption \ref{asm:correlation_basic}.
For property \ref{prp:cond_indep}, fix some $i \in [d], j \in [n_i/2], \vp \in \{\pm 1\}^2, y' \in \{\pm 1\}$, such that:
\[\prob{(\vx, y) \sim \D{i}}{\gamma_{i-1,j}(x_{2j-1}, x_{2j}) = \gamma_{i-1,j}(\vp)} > 0
\]
Assume w.l.o.g. that $j = 1$. Denote by $W$ the set of all possible choices for $x_3, \dots, x_{n_i}$, such that when $(x_1, x_2) = \vp$, the resulting label is $y'$. Formally:
\[
W := \left\lbrace (x_3, \dots, x_{n_i}) ~:~
\Gamma_{i \dots d}(p_1, p_2, x_3, \dots, x_{n_i}) = y'\right\rbrace
\]
Then we get:
\begin{align*}
&\prob{\D{i}}{(x_1, x_2) = \vp, y = y', \gamma_{i-1,j}(x_1, x_2) = \gamma_{i-1,j}(\vp)} \\
&= \prob{\D{i}}{(x_1, x_2) = \vp, (x_3, \dots, x_{n_i}) \in W, \gamma_{i-1,j}(x_1, x_2) = \gamma_{i-1,j}(\vp)} \\
&= \prob{\D{i}}{(x_1, x_2) = \vp,\gamma_{i-1,j}(x_1, x_2) = \gamma_{i-1,j}(\vp)} \cdot \prob{\D{i}}{(x_3, \dots, x_{n_i}) \in W} \\
&= \prob{\D{i}}{(x_1, x_2) = \vp | \gamma_{i-1,j}(x_1, x_2) = \gamma_{i-1,j}(\vp)} \cdot \prob{\D{i}}{\gamma_{i-1,j}(x_1, x_2) = \gamma_{i-1,j}(\vp),(x_3, \dots, x_{n_i}) \in W} \\
&= \prob{\D{i}}{(x_1, x_2) = \vp | \gamma_{i-1,j}(x_1, x_2) = \gamma_{i-1,j}(\vp)} \cdot \prob{\D{i}}{y=y',\gamma_{i-1,j}(x_1, x_2) = \gamma_{i-1,j}(\vp)} \\
\end{align*}
And dividing by $\prob{\D{i}}{\gamma_{i-1,j}(x_1, x_2) = \gamma_{i-1,j}(\vp)}$ gives the required.

For property \ref{prp:cond_bound}, we observe two cases.
If $c_{i,j} \ge 0$ then:
\begin{align*}
\Delta \le c_{i,j}- \mean{}{y} &= \mean{}{x_j y - y} = \mean{}{y(x_j - 1)} \\
&= 2\prob{}{x_j = -1 \wedge y = -1}- 2\prob{}{x_j = -1 \wedge y = 1} \\
&\le 2 \prob{}{x_j = -1 \wedge y = -1}
\le 2 \prob{}{x_j = -1}
\end{align*}
Otherwise, if $c_{i,j} < 0$ we have:
\begin{align*}
\Delta \le -c_{i,j}- \mean{}{y} &= \mean{}{-x_j y - y} = -\mean{}{y(x_j + 1)} \\
&= 2\prob{}{x_j = 1 \wedge y = -1}- 2\prob{}{x_j = 1 \wedge y = 1} \\
&\le 2 \prob{}{x_j = 1 \wedge y = -1}
\le 2 \prob{}{x_j = 1}
\end{align*}
So, in any case $\prob{}{x_j = 1} \in (\frac{\Delta}{2}, 1-\frac{\Delta}{2})$, and since every bit in every layer is independent, we get property \ref{prp:cond_bound} holds with $\epsilon = \frac{\Delta^2}{4}$.
\end{proof}

\section{Proofs of Section \ref{sec:distributions}}
\begin{proof} of \lemref{lem:parity_expressivity}.

For every gate $(i,j)$, let $J_{i,j}$ be the subset of leaves in the binary tree whose root is the node $(i,j)$. Namely, $J_{i,j} := \{(j-1) 2^{d-i}+1, \dots, j2^{d-i}\}$. Then we show inductively that for an input $\vx \in \{\pm 1\}^n$, the $(i,j)$ gate outputs: $\prod_{l \in I \cap J_{i,j}} x_l$:
\begin{itemize}
\item For $i = d-1$, this is immediate from the definition of the gate $\gamma_{d-1, j}$.
\item Assume the above is true for some $i$ and we will show this for $i-1$. By definition of the circuit, the output of the $(i-1,j)$ gate is a product of the output of its inputs from the previous layers, the gates $(i,2j-1)$, $(i, 2j)$. By the inductive assumption, we get that the output of the $(i-1,j)$ gate is therefore:
\[
\left( \prod_{l \in J_{i,2j-1} \cap I} x_l \right) \cdot \left( \prod_{l \in J_{i,2j} \cap I} x_l \right) = \prod_{l \in (J_{i,j2-1} \cup J_{i,2j}) \cap I} x_l
= \prod_{l \in J_{i-1,j}} x_l
\]
\end{itemize}
From the above, the output of the target circuit is $\prod_{l \in J_{0,1} \cap I} x_l = \prod_{l \in I} x_l$, as required.
\end{proof}

\begin{proof} of \lemref{lem:parity_product_distribution}.

By definition we have:
\[
c_{i,j} = \mean{(\vx, y) \sim \mathcal{D}}{\Gamma_{(i+1) \dots d}(\vx)_j y} = \mean{(\vx, y) \sim \mathcal{D}}{\Gamma_{(i+1) \dots d}(\vx)_j y}
= \mean{(\vx, y) \sim \mathcal{D}}{\Gamma_{(i+1) \dots d}(\vx)_jx_1 \cdots x_k}
\]
Since we require $\mathcal{I}_{i,j} \ne 0$, then we cannot have 
$\Gamma_{(i+1) \dots d}(\vx)_{j} \equiv 1$.
So, from what we showed previously, it follows that $\Gamma_{(i+1) \dots d}(\vx)_{j} = \prod_{j' \in I'} x_{j'}$ for some $\emptyset \ne I' \subseteq I$.
Therefore, we get that: 
\[
c_{i,j} = \mean{\mathcal{D}}{\prod_{j' \in I \setminus I'} x_{j'}} = \prod_{j' \in I \setminus I'} \mean{\mathcal{D}}{ x_{j'}}
= \prod_{j' \in I \setminus I'} (2p_{j'}-1)
\]
Furthermore, we have that:
\[
\mean{\mathcal{D}}{y} = \mean{\mathcal{D}}{\prod_{j' \in I} x_{j'}}
= \prod_{j' \in I} \mean{\mathcal{D}}{x_{j'}}
= \prod_{j' \in I} (2p_{j'}-1)
\]
And using the assumption on $p_j$ we get:
\begin{align*}
\abs{c_{i,j}} - \abs{\mean{\mathcal{D}}{y}}
&= \prod_{j' \in [k] \setminus I'} \abs{2p_{j'}-1}
-\prod_{j' \in [k]} \abs{2p_{j'}-1} \\
&= \left(\prod_{j' \in [k] \setminus I'} \abs{2p_{j'}-1} \right) \left(1 - \prod_{j' \in I'} \abs{2p_{j'}-1}\right) \\
&\ge \left(\prod_{j' \in [k] \setminus I'} \abs{2p_{j'}-1} \right) \left(1 - (1-2\xi)^{\abs{I'}}\right) \\
&\ge (2\xi)^{k-\abs{I'}} \left( 1- (1-2 \xi) \right) \ge (2\xi)^k \\
\end{align*}

Now, for the second result, we have:
\begin{align*}
\prob{(\vz,y) \sim \Gamma_{i \dots d} (\mathcal{D})}{z_j = 1} 
&= \mean{(\vx, y) \sim \mathcal{D}}{\1\{\Gamma_{(i+1) \dots d}(\vx)_j=1\}} \\
&= \mean{(\vx, y) \sim \mathcal{D}}{\frac{1}{2} (\prod_{j' \in I'} x_{j'} + 1)}\\
&= \frac{1}{2}\prod_{j' \in I'} \mean{(\vx, y) \sim \mathcal{D}}{x_{j'}} + \frac{1}{2}
\end{align*}
And so we get:
\begin{align*}
\abs{\prob{(\vz,y) \sim \Gamma_{i \dots d} (\mathcal{D})}{z_j = 1}  - \frac{1}{2}}
&= \frac{1}{2}\prod_{j' \in I'} \abs{\mean{(\vx, y) \sim \mathcal{D}}{x_{j'}}} \\
&< \frac{1}{2} (1-2\xi)^{\abs{I'}} \le \frac{1}{2} - \xi
\end{align*}
\end{proof}

\begin{proof} of \lemref{lem:correlation}
For every $i \in [d]$ and $j \in [2^i]$, denote the following:
\[
p^{+}_{i,j} = \prob{(\vx,y) \sim \D{i}}{x_j = 1 | y = 1},
~p^{-}_{i,j} = \prob{(\vx,y) \sim \D{i}}{x_j = 1 | y = -1}
\]
Denote $\mathcal{D}^{(i)}|_{\vz}$ the distribution
$\mathcal{D}^{(i)}$ conditioned on some fixed value $\vz$ sampled from $\mathcal{D}^{(i-1)}$.
We prove by induction on $i$ that $|p_{i,j}^{+} - p_{i,j}^{-}| = \left( \frac{2}{3} \right)^i$:
\begin{itemize}
\item For $i = 0$ we have $p^{+}_{i,j} = 1$ and $p^{-}_{i,j} = 0$, so the required holds.
\item Assume the claim is true for $i-1$, and notice that we have for every
$\vz \in \{\pm 1\}^{2^{i-1}}$:
\begin{align*}
\prob{(\vx,y) \sim \D{i}}{x_j = 1 | y=1}
&= \prob{(\vx,y) \sim \D{i}|_{\vz}}{x_j = 1 | z_{\lceil j/2 \rceil} = 1} 
\cdot \prob{(\vz,y) \sim \D{i-1}}{z_{\lceil j/2 \rceil} = 1 | y=1} \\
&+ \prob{(\vx,y) \sim \D{i}|_{\vz}}{x_j = 1 | z_{\lceil j/2 \rceil} = -1} 
\cdot \prob{(\vz,y) \sim \D{i-1}}{z_{\lceil j/2 \rceil} = -1 | y=1} \\
&=\begin{cases}
p_{i-1,\lceil j/2 \rceil}^+ + \frac{1}{3}(1-p_{i-1,\lceil j/2 \rceil}^+) & if~
\gamma_{i-1,\lceil j/2 \rceil} = \wedge \\
\frac{2}{3} p_{i-1,\lceil j/2 \rceil}^+ & if~ \gamma_{i-1,\lceil j/2 \rceil} = \vee \\
\frac{1}{3}p_{i-1,\lceil j/2 \rceil}^+ + (1-p_{i-1,\lceil j/2 \rceil}^+) & if~
\gamma_{i-1,\lceil j/2 \rceil} = \neg \wedge \\
\frac{2}{3} (1-p_{i-1,\lceil j/2 \rceil}^+) & if~ \gamma_{i-1,\lceil j/2 \rceil} = \neg \vee \\
\end{cases} \\
&=\begin{cases}
\frac{2}{3}p_{i-1,\lceil j/2 \rceil}^+ - \frac{1}{3} & if~
\gamma_{i-1,\lceil j/2 \rceil} = \wedge \\
\frac{2}{3} p_{i-1,\lceil j/2 \rceil}^+ & if~ \gamma_{i-1,\lceil j/2 \rceil} = \vee \\
1-\frac{2}{3} p_{i-1,\lceil j/2 \rceil}^+ & if~
\gamma_{i-1,\lceil j/2 \rceil} = \neg \wedge \\
\frac{2}{3}-\frac{2}{3}p_{i-1,\lceil j/2 \rceil}^+ & if~ \gamma_{i-1,\lceil j/2 \rceil} = \neg \vee \\
\end{cases} \\
\end{align*}
Similarly, we get that:
\[
\prob{(\vx,y) \sim \D{i}}{x_j = 1 | y=-1}
=\begin{cases}
\frac{2}{3}p_{i-1,\lceil j/2 \rceil}^- - \frac{1}{3} & if~
\gamma_{i-1,\lceil j/2 \rceil} = \wedge \\
\frac{2}{3} p_{i-1,\lceil j/2 \rceil}^- & if~ \gamma_{i-1,\lceil j/2 \rceil} = \vee \\
1-\frac{2}{3} p_{i-1,\lceil j/2 \rceil}^- & if~
\gamma_{i-1,\lceil j/2 \rceil} = \neg \wedge \\
\frac{2}{3}-\frac{2}{3}p_{i-1,\lceil j/2 \rceil}^- & if~ \gamma_{i-1,\lceil j/2 \rceil} = \neg \vee \\
\end{cases}
\]
Therefore, we get:
\[
|p_{i,j}^+-p_{i,j}^-| =
\frac{2}{3}|p_{i-1,\lceil j/2 \rceil}^+ - p_{i-1,\lceil j/2 \rceil}^-|
= \left(\frac{2}{3} \right)^i
\]
\end{itemize}

From this, we get:
\begin{align*}
\abs{\mean{(\vx,y) \sim \D{i}}{x_j y}}
&= \abs{\mean{(\vx,y) \sim \D{i}}{(2\1\{x_j = 1\}-1) y}} \\
&= \abs{2\mean{(\vx,y) \sim \D{i}}{\1\{x_j = 1\}y}-\mean{}{y}} \\
&= \abs{2\left( \prob{\D{i}}{x_j = 1, y= 1}- \prob{\D{i}}{x_j = 1, y= -1}\right)} \\
&= \abs{2\left( p_{i,j}^+ \prob{}{y=1}- p_{i,j}^- \prob{}{y=-1}\right)} \\
&= \abs{p_{i,j}^+ -p_{i,j}^-} = \left( \frac{2}{3} \right)^d
\end{align*}
And hence:
\begin{align*}
\abs{\mean{(\vx,y) \sim \D{i}}{x_j y}} - \abs{\mean{(\vx,y) \sim \D{i}}{y}}
\ge \left( \frac{2}{3} \right)^d
\end{align*}
\end{proof}

\begin{proof} of \lemref{lem:gate_distribution}
Fix some $\vz' \in \{\pm 1\}^{n_i/2}$ and $y' \in \{\pm 1\}$. Then we have:
\begin{align*}
\prob{(\vx,y) \sim \Gamma_i(\D{i})}{(\vx,y) = (\vz', y')}
&= \prob{(\vx,y) \sim \D{i}}{(\Gamma_i(\vx),y) = (\vz', y')} \\
&= \prob{(\vx,y) \sim \D{i}}{\forall j~ \gamma_{i-1,j}(x_{2j-1}, x_{2j}) = z'_j~and~y = y'} \\
&= \prob{(\vz,y) \sim \D{i-1}}{(\vz,y) = (\vz', y')} \\
\end{align*}
By the definitions of $\D{i}$ and $\D{i-1}$.
\end{proof}

\section{Proofs of Section \ref{sec:depth_separation}}
We start by showing that a depth-two network with bounded integer weights has sign-rank that grows only polynomially in the input dimension:
\begin{lemma}
\label{lem:integer_weights}
Let $g(\vx,\vy) = \sign \left( \sum_{i=1}^k u_i \sigma(\inner{\vw_i,\vx} + \inner{\vv_i, \vy} + b_i) \right)$ for some activation function $\sigma$, some integer weights $\vw_i, \vv_i \in [-B,B]^{n'}$ and some bounded integer bias $b_i \in [-B,B]$.
Then $\mM = [g(x,y)]_{x,y}$ has $sr(\mM) \le 4Bkn'$.
\end{lemma}
\begin{proof} of \lemref{lem:integer_weights}
Denote $g_i(\vx,\vy) = \sigma(\inner{\vw_i, \vx} + \inner{\vv_i,\vy})$, and let $\mM_i = [g_i(\vx,\vy)]_{\vx,\vy}$. Notice that $\inner{\vw_i, \vx} \in [-Bn', Bn'] \cap \integers$ and also $\inner{\vv_i,\vx} \in [-Bn', Bn'] \cap \integers$. So, we can order the rows in $\mM_i$ by the value of $\inner{\vw_i, \vx}$, and the columns of $\mM_i$ by the value of $\inner{\vv_i,\vx}$, and we get that the new matrix is composed of constant sub-matrices, each one taking values in $[-2Bn', 2Bn'] \cap \integers$, so we have $\rank(\mM_i) \le 4Bn'$. Therefore, from sub-additivity of the rank:
\[
\rank(\sum_i u_i \mM_i) \le \sum_i \rank(\mM_i) \le 4Bkn'
\]
And by definition of the sign-rank we get:
\[
sr(\mM) \le \rank(\sum_i u_i \mM_i) \le 4Bkn'
\]
\end{proof}

Using this, we can show that any depth-two network that agrees with $f_m$ with some margin, has exponential size:
\begin{proof} of \thmref{thm:depth_separation}
Fix some $N \in \naturals$. Let $\hat{\vw}_i = \lceil \vw_i \cdot N \rceil, \hat{\vv}_i = \lceil N \cdot \vv_i \rceil, \hat{b}_i = \lceil b_i \cdot N \rceil$, with $\lceil \cdot \rceil$ taken coordinate-wise. Then, we have: $\norm{\frac{1}{N}\hat{\vw}_i - \vw_i}_\infty, \norm{\frac{1}{N} \hat{\vv}_i -\vv_i}_\infty, \norm{\frac{1}{N} \hat{b}_i - b_i}_\infty \le \frac{1}{N}$. Denote:
\[
\hat{g}(\vx,\vy) = \sum_{i=1}^k \frac{u_i}{N} \sigma(\inner{\hat{\vw}_i, \vx} + \inner{\hat{\vv}_i, \vy} + \hat{b}_i)
\]
Then, for every $\vx,\vy \in \mathcal{X}'$ we have:
\begin{align*}
\abs{\hat{g}(\vx,\vy) - g(\vx,\vy)}
&= \abs{\sum_{i=1}^k \frac{u_i}{N} \sigma(\inner{\hat{\vw}_i, \vx} + \inner{\hat{\vv}_i, \vy} + \hat{b}_i) - \sum_{i=1}^k u_i \sigma(\inner{\vw_i, \vx} + \inner{\vv_i, \vy} + b_i)} \\
&\le \sum_{i=1}^k \abs{u_i} \abs{\sigma(\frac{1}{N}\inner{\hat{\vw}_i, \vx} + \frac{1}{N}\inner{\hat{\vv}_i, \vy} + \frac{1}{N}\hat{b}_i) - \sigma(\inner{\vw_i, \vx} + \inner{\vv_i, \vy} + b_i)} \\
&\le B \sum_{i=1}^k \abs{\inner{\frac{1}{N}\hat{\vw}_i - \vw_i, \vx}}
+ \abs{\inner{\frac{1}{N}\hat{\vv}_i - \vv_i, \vx}} + \abs{\frac{1}{N} \hat{b}_i - b_i} \\
&\le B \sum_{i=1}^k \norm{\frac{1}{N}\hat{\vw}_i - \vw_i}_2 \norm{\vx}_2
+ \norm{\frac{1}{N}\hat{\vv}_i - \vv_i}_2 \norm{\vx}_2 + \abs{\frac{1}{N} \hat{b}_i - b_i} \\
&\le Bk \frac{2n'+1}{N}
\end{align*}
Then, if $N = \left\lceil\frac{1}{\gamma}2Bk(2n'+1) \right\rceil$ we get that $\abs{\hat{g}(\vx,\vy)-g(\vx,\vy)} \le \frac{\gamma}{2}$. From our assumption on $g$, we get that $\sign\hat{g}(\vx,\vy) = \sign g(\vx,\vy) = f_m(\vx,\vy)$. Now, since $\hat{g}$ has integer weights and biases bounded by $N \cdot B$ in the first layer, from Lemma \ref{lem:integer_weights} we get that for $\mM = [\sign \hat{g}(x,y)]_{x,y} = [f_m(x,y)]_{x,y}$
we have $sr(\hat{\mM}) \le 4 Bkn' \lceil \frac{1}{\gamma} 2Bk(2n'+1) \rceil \le 24 B^2k^2 {(n')}^2$. From Theorem \ref{thm:sign_rank_ac0} we have $sr(M) = 2^{\Omega(m)}$. Recall that $n' = m^3$, and so we get that $k \ge 2^{\Omega(m)}$.
\end{proof}

\end{document}